\documentclass{article}

\usepackage{header_arxiv}
\usepackage{commands}

\title{Convergence and Sample Complexity of First-Order Methods \\ for Agnostic Reinforcement Learning}

\author{
Uri Sherman%
\thanks{Blavatnik School of Computer Science, Tel Aviv University; \texttt{urisherman@mail.tau.ac.il}.
}
\and
Tomer Koren%
\thanks{Blavatnik School of Computer Science, Tel Aviv University, and Google Research; \texttt{tkoren@tauex.tau.ac.il}.}
\and
Yishay Mansour%
\thanks{Blavatnik School of Computer Science, Tel Aviv University, and Google Research; \texttt{mansour.yishay@gmail.com}. 
}
}

\begin{document}
\maketitle

\begin{abstract}
    We study reinforcement learning (RL) in the agnostic policy learning setting, 
    where the goal is to find a policy whose performance is competitive with the best policy in a given class of interest $\Pi$---crucially, without assuming that $\Pi$ contains the optimal policy.
    We propose a general policy learning framework that reduces this problem to first-order optimization in a non-Euclidean space, leading to new algorithms as well as shedding light on the convergence properties of existing ones.
    Specifically, under the assumption that $\Pi$ is convex and satisfies a variational gradient dominance (VGD) condition---an assumption known to be strictly weaker than more standard completeness and coverability conditions---we obtain sample complexity upper bounds for three policy learning algorithms:
    \emph{(i)} Steepest Descent Policy Optimization, derived from a constrained steepest descent method for non-convex optimization;
    \emph{(ii)} the classical Conservative Policy Iteration algorithm \citep{kakade2002approximately} reinterpreted through the lens of the Frank-Wolfe method, which leads to improved convergence results;
    and \emph{(iii)} an on-policy instantiation of the well-studied Policy Mirror Descent algorithm.
    Finally, we empirically evaluate the VGD condition across several standard environments, demonstrating the practical relevance of our key assumption.
\end{abstract}

% \begin{refsection}
\section{Introduction}
Policy Optimization (PO) algorithms are a class of methods in Reinforcement Learning (RL; \citealp{sutton2018reinforcement, mmt2022rlbook}) in which an agent's policy is iteratively updated to minimize long-term cost, as defined by the environment's value functions.
Modern applications of PO methods \citep[e.g.,][]{lillicrap2015continuous, schulman2015high, akkaya2019solving, ouyang2022training} often involve large-scale environments that lack well-defined structure, and by that require function approximation techniques
in order to learn efficiently.
Typically, PO algorithms represent the agent’s policy using neural network models—commonly referred to as \emph{actor} networks. Notably, these setups are inherently agnostic: the learner searches for an assignment of network parameters that is competitive with the best achievable under the model, without any guarantee that the optimal policy is expressible by the actor architecture.

Motivated by this, we consider the problem of \emph{agnostic policy learning} in the general function approximation setup \citep{kakade2003sample,krishnamurthy2025role},
where the learner is given optimization oracle access to a policy class $\Pi$ and is required to find a policy that performs nearly as well as the best in-class policy. 
It is well known that $\Pi$-completeness and coverage conditions allow for sample efficient policy learning \citep{agarwal2019reinforcement,agarwal2021theory,bhandari2024global},\footnote{Roughly speaking, $\Pi$-completeness is defined as closure to policy improvement steps, and coverage as a constant upper bound on the worst case ratio between the initial state distribution and the optimal policy occupancy measure.}
however, completeness implies realizability, and both conditions are generally deemed too strong to hold in practice. Furthermore, the extent to which they hold or not is hard to measure empirically. 

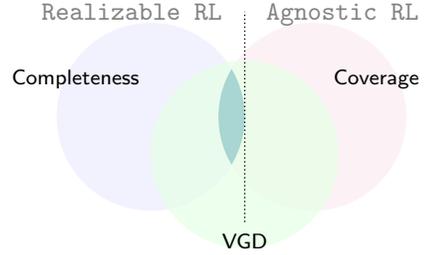
\begin{wrapfigure}{r}{0.45\textwidth}
  \centering
  \begin{tikzpicture}[scale=0.5, line width=0.1pt, font=\sffamily\footnotesize]

\begin{scope}[shift={(1cm,0cm)}]
\node[font=\ttfamily, text=gray] at (-3,2.8){Realizable RL};
\node[font=\ttfamily, text=gray] at (2.6,2.73){Agnostic RL};

  \fill[blue!5] (-2.5,0) circle (2.5);
  \fill[purple!5] (1.8,0) circle (2.5);
  \fill[green!10, opacity=0.7] (0,-1) circle (2.5);

  % Intersection area with custom color
%  \begin{scope}
%    \clip (-2,0) circle (2);
%    \fill[blue!20, opacity=0.8] (0.5,0) circle (2);
%  \end{scope}
%  
  \begin{scope}
    \clip (-2.5,0) circle (2.5);
    \clip (1.8,0) circle (2.5);
    \fill[bluegreen!40, opacity=0.8] (0,-1) circle (2.7);
  \end{scope}

  % Borders
%  \draw (0,0) circle (1.5);
%  \draw (1.5,0) circle (1.5);
%  \draw (0,-1) circle (1.8);
  
  % Labels
  \node at (-4.5, 1) {Completeness};
  \node at (3.5, 1) {Coverage};
  \node at (0, -3.3) {VGD};
  
 \draw[densely dotted, line width=.5pt] (0,-2.8) -- (0,2.8);
\end{scope}
\end{tikzpicture}
  \caption{
        \protect\tikz[baseline=0.4ex]
		\protect\fill[bluegreen, opacity=0.5] (0,0) rectangle (2ex,2ex);
		 Completeness + coverage allows for sample efficient policy learning \citep{kakade2002approximately}. These conditions imply in particular, realizability.
    	\protect\tikz[baseline=0.4ex]
		\protect\fill[bluegreen, opacity=0.5] (0,0) rectangle (2ex,2ex);
		\protect\tikz[baseline=0.4ex]
		\protect\fill[green!20, opacity=1] (0,0) rectangle (2ex,2ex);
		VGD allows for sample efficient learning, and in particular accommodates agnostic (non-realizable) setups.
  }
  \label{fig:vgd_completeness_coverage}
\end{wrapfigure}
In this work, we adopt instead the assumption that $\Pi$ satisfies a variational gradient dominance condition \citep{agarwal2021theory,xiao2022convergence, bhandari2024global}, which is known to be strictly weaker than completeness and coverage, and in particular, may accommodate non-realizable setups \citep{bhandari2024global,sherman2025convergence} (see \cref{fig:vgd_completeness_coverage}, and \cref{sec:comp_prior_art} for further details). Furthermore, the VGD parameters are to a degree measurable in practice, and appear better suited to characterize convergence of first-order policy learning algorithms. Indeed, the empirically observed parameters are reasonable compared to the theoretical, hard to measure ones associated completeness and coverage; and the VGD assumption pinpoints the precise properties required in convergence analysis under completeness and coverage.

\subsection{Our contributions}
In this work, we make the following contributions.
\paragraph{Policy learning framework.}
We introduce a natural policy learning framework that reduces Agnostic RL to first order optimization in a constrained non-convex non-Euclidean setup.
Consequently, we obtain practical, sample efficient%
\footnote{By sample efficient, we refer to methods with convergence bounds that scale with the log-covering number of the policy class, but not with the cardinality of the state space.} 
policy learning algorithms, and along the way also improvements in state-of-the-art iteration complexity upper bounds. Importantly, our framework and reduction are completely independent of the choice of policy class parametrization.
In the function approximation policy learning setup, the primary way by which policies are produced is by constructing an objective function $\smash{\widehat \Phi_k} \colon \Pi \to \R$, and invoking an optimization oracle to compute an approximate minimizer:
\begin{align}\label{def:hat_phi_k_0}
    \pi^{k+1}
    \gets \argmin\nolimits_{\pi \in \Pi}
        \widehat \Phi_k (\pi)
    .
\end{align}
Roughly speaking, our reduction makes use of the policy gradient theorem \citep{sutton1999policy} in the direct parametrization case, along with an on-policy estimation scheme and a standard concentration argument to yield that with a suitable choice of $\smash{\widehat \Phi_k}$, \cref{def:hat_phi_k_0} produces a gradient step w.r.t.~the value function in policy (i.e., state-action, functional) space. This, combined with the local-smoothness property of the value function \citep{sherman2025convergence}, implies the algorithm may be cast as a first-order method taking gradient steps on a smooth objective. Crucially, unlike Euclidean smoothness used in prior work \citep[e.g.,][]{agarwal2021theory}, this leads to rates that are independent of the size of the state space. Furthermore, it is substantially different than operating in the parameter space of the policy parametrization \citep[e.g., ][]{mei2020global, yuan2022general, bhandari2024global}.

\paragraph{Non-Euclidean smooth constrained optimization.} 
We highlight smooth constrained non-convex optimization in a non-Euclidean space as the principal setting to which agnostic RL reduces to. In this context, we provide novel analyses that, to our knowledge, have not appeared in the literature previously, including (i) a steepest descent method for smooth non-convex constrained optimization, and (ii) an analysis for an approximate Frank-Wolfe that holds for VGD objectives (a weaker condition compared to convexity).
The constrained steepest descent method which we analyze here is a natural generalization of gradient descent to objectives smooth w.r.t.~a non-Euclidean norm. \citet{kelner2014almost} appear to be the first to analyze the method for convex functions in the unconstrained setting, while \citet{xiao2022convergence} provides an analysis for VGD objectives in the constrained Euclidean setup. Our work is the first to further extend the method to the constrained, non-Euclidean setup, using a non-standard notion of steepest descent direction w.r.t.~a constrained set of potential gradient mappings.
The optimization setup and our analyses naturally accommodate also \emph{local smoothness} of the objective function, which is crucial for the interesting cases of the reduction mentioned in the preceding paragraph.

\paragraph{Iteration and sample complexity for policy learning algorithms.}
Combining the elements above, and assuming the policy class satisfies a VGD condition and is convex, we obtain upper bounds on the sample complexity of several algorithms within our proposed framework. 
In particular, we propose a sample efficient \textbf{Steepest Descent Policy Optimization (SDPO)} method, based on a constrained steepest descent method for non-convex, non-Euclidean optimization which we analyze in this work for the first time.
We then revisit the classic \textbf{Conservative Policy Iteration} (CPI; \citealp{kakade2002approximately}) algorithm, and cast it as an instance of the well-known Frank-Wolfe \citep{frank1956algorithm} algorithm, leading to (i) improved iteration complexity, and (ii) a variant of CPI---Doubly Approximate CPI (DA-CPI)---which is sample efficient \emph{and} more practical, as we explain shortly in the discussion that follows.
Finally, we establish polynomial sample complexity for the well studied \textbf{Policy Mirror Descent} \citep{tomar2020mirror,  xiao2022convergence, lan2023policy} algorithm. To the best of our knowledge, our work is the first to obtain sample complexity upper bounds for PMD that are independent of the policy class parametrization. 
\begin{table}[t]
\centering
\small
\renewcommand{\arraystretch}{2.2}
\begin{tabular}{
	@{} >{\centering\arraybackslash}p{4.5cm} 
		>{\centering\arraybackslash}p{5.5cm} 
    	>{\centering\arraybackslash}p{0.6cm} 	
            >{\centering\arraybackslash}p{0.6cm} 
		>{\centering\arraybackslash}p{0.6cm} 
		>{\centering\arraybackslash}p{0.6cm} 
	}
\toprule
\textbf{Method} &
\textbf{Rate} &
\textbf{GP} &
\textbf{Ag} &
\textbf{NC} &
\textbf{AOE} \\
\midrule

\makecell{CPI$^{\,\rm a}$ \\ \citep{kakade2002approximately}} &
\(\displaystyle 
\tfrac{D_\infty}{\sqrt{K}}
\;+\;
\mathcal{E}(\Pi)\,D_\infty\) &
\cmark & \xmark & \cmark & \xmark 
\\

\makecell{Log-linear NPG$^{\,\rm b}$ \\ \citep{agarwal2021theory}} &
\(\displaystyle
\tfrac{1}{\sqrt K}
\;+\;
\sqrt{ \kappa\varepsilon + D_\infty \varepsilon_{\rm approx}}\) &
\xmark & \xmark & \cmark &  \cmark 
\\
\makecell{Log-linear NPG$^{\,\rm c}$ \\ \citep{yuan2023loglinear}} &
\(\displaystyle
\br{1-\tfrac{1}{D_\infty}}^K
\;+\;
D_\infty\sqrt{C_v (\varepsilon +  \varepsilon_{\rm approx})}\) 
&
\xmark & \xmark & \cmark & \cmark 
\\
\makecell{PMD$^{\,\rm d}$ \\  \citep{alfano2023novel} }&
\(\displaystyle
(1-\tfrac1D_\infty)^K
\;+\;
 D_\infty \sqrt{C_v\,\varepsilon_{\mathrm{pmdc}}}\) &
\cxmark & \xmark & \cmark & \cmark 
\\

\makecell{PMD \\ \citep{sherman2025convergence}} &
\(\displaystyle
\tfrac{\nu^{2}}{K^{2/3}}
\;+\;
(\cvgd + K^{1/6})\,\varepsilon^{1/4}
\;+\;
\varepsilon_{\mathrm{vgd}}\) &
\cmark & \cmark & \cxmark & \cmark 
\\

\midrule 

\makecell{SDPO \\ \textbf{(This Work)}} &
\(\displaystyle
\tfrac{\cvgd^{2}}{K^{2/3}}
\;+\;
\cvgd K^{1/6}\,\sqrt{\varepsilon}
\;+\;
\varepsilon_{\mathrm{vgd}}\) &
\cmark & \cmark & \xmark & \cmark 
\\

\makecell{CPI$^{\,\rm e}$ \\ \textbf{(This Work)}} &
\(\displaystyle
\tfrac{\cvgd^{2}}{K}
\;+\;
\cvgd\,\varepsilon
\;+\;
\varepsilon_{\mathrm{vgd}}\) &
\cmark & \cmark & \cmark & \xmark 
\\

\makecell{DA-CPI \\ \textbf{(This Work)}} &
    \(\displaystyle
    {\tfrac{\cvgd^2}{K^{2/3}} 
    	+ \cvgd^2\varepsilon^{2/3}  K^{2/3}}
        +\epsvgd
    \) 
&
\cmark & \cmark & \xmark & \cmark 
\\
\bottomrule
\\
\end{tabular}
\vspace{-3ex}
\caption{
Comparison of different policy optimization algorithms.
\textbf{Rate:} Gives the suboptimality after $K$ iterations, as a function of error terms.
$\varepsilon$ denotes the value fitting error under the relevant sampling distribution, and $D_\infty$ the distribution mismatch coefficient. $\cbr{\cE(\Pi), \varepsilon_{\rm approx}, \varepsilon_{\rm pmdc}}$ all measure some form of ``completeness error''; refer to \cref{sec:comp_prior_art} for further details.
\textbf{GP:} Whether the method applies for general parameterizations of $\Pi$;
\textbf{Ag:} Whether it gives meaningful guarantees in an agnostic setting under the VGD assumption (replace $\nu \to D_\infty$ to compare with non-agnostic bounds, error floors are similar);
\textbf{NC:} Whether it applies for Non-Convex $\Pi$;
\textbf{AOE:} Whether it is actor-oracle efficient.
\textbf{[a, b, c, d, e]:} [a] The dependence on the completeness error is given in \citep{scherrer2014local}. [b] $\kappa$ relates to an eigenvalue condition of the state-action features covariance under the learning distribution. [b+c] Both works also provide bounds in terms of a bias error which we do not include here. [c+d] These rates require geometrically increasing step sizes, for constant step sizes sub-linear rates exist.
[e] Our version of CPI makes a different choice of step sizes.
}
\vspace{-3ex}
\label{tab:po_algs_comparison}
\end{table}

\cref{tab:po_algs_comparison} gives a detailed comparison between several algorithms of interest.
In particular, our SDPO bound provides a substantial improvement over PMD in the agnostic setting \citep{sherman2025convergence}, by obtaining $\sqrt \varepsilon$ error dependence rather than $\smash{\varepsilon^{1/4}}$. When converting the result to a sample complexity upper bound this becomes significant. Furthermore, SDPO obtains better dependence on the action set cardinality $A$ when applied with the $L^1$ action norm, thereby lifting one of the two barriers left in \citet{sherman2025convergence} to obtain rates for large action spaces (rates that scale at most logarithmically with $A$). The same is true for DA-CPI, which also improves upon the guarantee of PMD in the agnostic setting in a similar fashion.

While our new bound for CPI provides an even sharper improvement, it is important to note that CPI in its original form is not as practical as the other algorithms we consider, in the following sense.
% In order to distinguish the practical from the less practical algorithms, 
Let us call a policy obtained by an invocation to the optimization oracle, such as $\pi^{k+1}$ in \cref{def:hat_phi_k_0}, an \emph{actor}. 
We say a policy learning algorithm is actor-oracle efficient (AOE) if the following two conditions hold. (i) The objective functions $\smash{\widehat \Phi_k}$ given to the optimization oracle can be evaluated in time that is independent of the size of the state space, and polynomial in other problem parameters.\footnote{The intention is that for any fixed $\pi$, $\widehat \Phi_k(\pi)$ is computable efficiently, that is, independently of the state space cardinality.} (ii) The actor space complexity---i.e., the maximal number of actors the algorithm requires to maintain in memory at any given time---is $O(1)$. Actor-oracle inefficient algorithms are generally not feasible (at least at the present time) for practical applications; for example, actor-memory linear in the number of iterations requires maintaining a prohibitively large amount of separate neural network models in memory. As we discuss further in \cref{sec:cpi}, CPI requires linear actor memory, while PMD, SDPO, and DA-CPI are all actor efficient, requiring at most two actor models at any given time.

\subsection{Related work}
There is a rich line of work that studies the PMD algorithm \citep{agarwal2021theory,xiao2022convergence,lan2023policy,alfano2023novel,ju2022policy}, however these, including \cite{sherman2025convergence}, either focus only on the optimization setup, or consider sample complexity subject to specific parametrization choices. The CPI algorithm was originally introduced by \citet{kakade2002approximately}, and its guarantees in terms of a completeness error was derived in \citet{scherrer2014local} (see also \citealp{agarwal2019reinforcement}).
Our work is directly inspired by \citet{sherman2025convergence}, 
where the connection of PMD to a constrained non-Euclidean optimization setup was recently established. Here, we take a more problem-centric view of agnostic RL and establish a broader connection between policy learning and optimization. Below, we survey related lines of work in more detail.

% Due to space constraints, we defer additional discussion of related work to \cref{sec:additional_related_work}.

\paragraph{Policy learning in the tabular or function approximation setups.}
Prototypical policy optimization methods in the tabular setup include variants of the PMD algorithm \citep{tomar2020mirror,xiao2022convergence,lan2023policy}, most commonly negative-entropy regularized PMD which is also known to be equivalent to the Natural Policy Gradient (NPG; \citealp{kakade2001natural}). Additional works that study PMD in the tabular setup include \citet{geist2019theory,lan2023policy,johnson2023optimal,zhan2023policy}. The modern analyses of PMD build on the early work of \citet{even2009online} and online mirror descent \citep{beck2003mirror,nemirovskij1983problem} and as such require some form of completeness of the policy class. An exception is the recent work of \citet{sherman2025convergence} where PMD is instead cast as a Bregman proximal point method, thus relaxing completeness conditions by relying instead on the VGD assumption.

The majority of recent works into policy optimization with function approximation focus on PMD variants \citep[e.g., ][]{ju2022policy,grudzien2022mirror,alfano2023novel,yuan2023loglinear} 
or policy gradients in parameter space \citep[e.g., ][]{zhang2020global,mei2020global,mei2021leveraging,yuan2022general,mu2024second}.
The influential works of \citet{bhandari2024global,agarwal2021theory} set the stage for modern research works both into PMD and policy gradients.
The notion of variational gradient dominance (or variants thereof) has appeared in several works in the context of policy learning, mostly in relation to policy gradient methods in parameter space or in the tabular setup \citep{mei2020global, agarwal2021theory,xiao2022convergence,bhandari2024global}.

The recent works of \citet{jia2023agnostic, krishnamurthy2025role} study the boundaries of PAC learnability of policy learning, focusing on forms of environment access, refined policy class conditions, and / or specific structural environment models such as the Block MDP \citep{krishnamurthy2016pac}. Notably, works on agnostic policy learning are comparatively scarce, and mostly focus on specific environment structures \citep{sekhari2021agnostic}.
More generally, there exist a myriad of works that study RL with function approximation (some of which may be classified as studies of policy learning) subject to particular environment structure, which we briefly review below.

\paragraph{Approximate policy iteration methods.}
Another class of prototypical policy learning algorithms directly to related our work are approximate policy iteration methods, which in particular include CPI \citep{kakade2002approximately}, as well as, for example, API \citep{bertsekas1996neuro} and PSDP \citep{bagnell2003policy}.
\citet{scherrer2014local} provide performance bounds for CPI subject to the completeness error (see also \citealp{agarwal2019reinforcement}).
To our knowledge, all results for approximate policy iteration methods require completeness of the policy class either directly or indirectly (by quantifying the error w.r.t.~to completeness).
In particular, this is true for the algorithms presented in the work of \citet{scherrer2014approximate}, which provides a thorough comparison of different approximate policy iteration schemes, as well as additional convergence bounds. These include an infinite horizon version of PSDP, and a faster rate for a variant of CPI based on line search and/or subject to stronger concentrability assumptions.
A bounded distribution mismatch coefficient $D_\infty$ is the weakest among forms of concentrability \citep{munos2003error,munos2005error,chen2019information}, and it too, as mentioned priorly, is deemed too strong to hold in large scale problems.
 Notably, the infinite horizon version of PSDP requires non-stationary policies and therefore does not fit into the policy learning model we consider here. In addition, the improved rates obtained for CPI require stronger concentrability assumptions and are therefore in applicable in the setting we consider here.

\paragraph{RL with function approximation more generally.}
There is a rich literature on RL with function approximation that focuses on setups where the environment exhibits some form of inherent structure \citep{jiang2017contextual,dong2020root,jin2020provably,jin2021bellman,du2021bilinear}.
One popular variant for which statistical and computational efficient policy learning is possible is the Linear MDP \citet{yang2019sample,yang2020reinforcement,jin2020provably}, and more generally the low-rank MDP \citep{jiang2017contextual,agarwal2020flambe} where the state-action feature are not given to the learner.
Our line of inquiry in this work aims at having no explicit structural assumptions on the environment. We adopt instead an optimization flavored assumption on the relation of the policy class to the landscape of the objective function, which turns out to be weaker than the standard completeness and coverage. As such, our work is better understood as extending the lines of work mentioned in the two preceding paragraphs.

\section{Preliminaries}
\paragraph{Discounted MDPs.}
A discounted MDP $\cM$ is defined by a tuple
	$\cM = (\cS, \cA, \P, r, \gamma, \rho_0)$,
where $\cS$ denotes the state-space, $\cA$ the action set, $\P\colon \cS \times \cA \to \Delta(\cS)$ the transition dynamics, $r \colon \cS \times \cA \to [0, 1]$ the regret (i.e., cost) function, $0<\gamma< 1$ the discount factor, and $\rho_0\in \Delta(\cS)$ the initial state distribution.
% \us{For notational convenience, for $s,a\in \cS\times\cA$ we let $\P_{s, a} \eqq \P(\cdot \mid s, a) \in \Delta(\cS)$ denote the next state probability measure.} 
We assume the action set is finite with $A\eqq |\cA|$, and identify $\R^A$ with $\R^\cA$. We additionally assume, for clarity of exposition and in favor of simplified technical arguments, that the state space is finite with $S\eqq |\cS|$, and identify $\R^\cS$ with $\R^S$. We further denote the effective horizon by $H\eqq \frac{1}{1-\gamma}$.
We emphasize that all our arguments may be extended to the infinite state-space setting with additional technical work.
An agent interacting with the MDP is modeled by a policy $\pi \colon \cS \to \Delta(\cA)$, for which we let $\pi_s \in \Delta(\cA)\subset\R^A$ denote the action probability vector at $s$ and $\pi_{s, a}\in [0,1]$ denote the probability of taking action $a$ at $s$. We denote the \emph{value} and Q-function by:
\begin{align*}
    V(\pi) \eqq  \E\sbr{\sum_{t=0}^\infty \gamma^t r(s_t, a_t) \mid s_0 \sim \rho_0, \pi}
    ; \quad 
    Q^\pi_{s, a} \eqq \E\sbr{\sum_{t=0}^\infty \gamma^t r(s_t, a_t) \mid s_0 = s, a_0=a, \pi}.
\end{align*}
We further denote the discounted state-occupancy measure of $\pi$ by $\mu^\pi$:
\begin{align*}
    \mu^\pi(s) \eqq 
    \br{1-\gamma}\sum_{t=0}^\infty \gamma^t \Pr(s_t = s \mid s_0 \sim \rho_0, \pi).
\end{align*}
It is easily verified that $\mu^\pi\in \Delta(\cS)$ is indeed a state probability measure.

\paragraph{Learning objective.}
We consider the problem of learning an approximately optimal policy within a given policy class $\Pi\subset \Delta(\cA)^\cS$:
\begin{align}\label{def:rl_opt_problem}
    \argmin_{\pi \in \Pi} V(\pi).
\end{align}
To avoid ambiguity, we denote the optimal value attainable by an in-class policy (a solution to \cref{def:rl_opt_problem}) by $\VstarPi$, and the optimal value attainable by any policy by $V^\star$:
\begin{align}
    \VstarPi \eqq \argmin_{\pi^\star \in \Pi}V(\pi^\star);
    % \; %icml_edit
    \quad
    V^\star \eqq \argmin_{\pi^\star \in \Delta(\cA)^\cS}V(\pi^\star).
\end{align}
Throughout this paper, we let $\pi^\star=\argmin_{\pi\in \Pi}V(\pi)$, and $\mu^\star \eqq \mu^{\pi^\star}$. The policy class $\Pi$ will always be clear from context. When we are in need to refer to the optimal policy / occupancy measure w.r.t. $\Pi_{\rm all}\eqq \Delta(\cA)^\cS$, we will say so explicitly.

\subsection{Problem setup}
We consider agnostic policy learning in the standard offline optimization oracle model. 
Given an objective function $\phi \colon \Pi \to \R$ an approximate minimizer may be produced by invoking the oracle. 
We will use the following notation for approximate minimization. For any set $\cX$ and objective $\phi \colon \cX \to \R$, we denote the set of $\epsilon$-approximate minimizers by:
\begin{align}
    \argmineps[\epsilon]_{x\in \cX}\cbr{\phi(x)}
    \eqq \cbr{
    x \in \cX \mid \phi(x) \leq \min_{x'}\phi(x') + \epsilon
    }.
\end{align}
When the agent decides to initiate a rollout episode in the environment, she starts at an initial state $s_0\sim \rho_0$, then for $t=0, \ldots, $ chooses action $a_t$ given $s_t$, incurs $r(s_t, a_t)$, and transitions to $s_{t+1} \sim \P(\cdot|s_t, a_t)$, until she decides to terminate the episode.
The sample complexity of a learning algorithm is the number of interaction time steps required to reach an $\epsilon$-approximate optimal policy, i.e., a policy $\hat \pi$ such that $V(\hat \pi) - V^\star(\Pi) \leq \epsilon$. 
As mentioned in the introduction, our key assumption is that $\Pi$ satisfies the following variational gradient dominance condition.
\begin{definition}[Variational Gradient Dominance]\label{def:vgd_mdp}
    We say that $\Pi$ satisfies a $(\cvgd, \epsvgd)$-variational gradient dominance (VGD) condition w.r.t.~a value function $V\colon \Delta(\cA)^\cS \to \R$, 
    if there exist constants $\cvgd, \epsvgd>0$, such that for any policy $\pi\in \Pi$:
    \begin{align}\label{eq:M_vgd}
        V(\pi) - \min_{\pi^\star\in \Pi} V(\pi^\star) 
        \leq 
        \cvgd \max_{\tilde \pi \in \Pi}\abr{\nabla V(\pi), \pi - \tilde \pi} + \epsvgd.
    \end{align}
\end{definition}

We conclude this section by repeating notations used throughout.
\begin{align*}
    \mu^k \eqq \mu^{\pi^k}, \quad 
    Q^k \eqq Q^{\pi^k}, \quad 
    S \eqq |\cS|, \quad
    A \eqq |\cA|, \quad
    H \eqq \frac{1}{1-\gamma}.
\end{align*}

\section{Policy learning via non-Euclidean constrained optimization}
\label{sec:framework_reduction}
In this section, we present our policy learning framework, and explain the reduction to first order optimization in a constrained, smooth non-Euclidean optimization setup. 
The reduction consists of the following three ingredients; 
\begin{enumerate}[label=(\roman*),leftmargin=!]
    \item Agnostic policy learning is cast as a first-order optimization problem over the policy (sometimes referred to as  ``functional'') space $\Pi$ that, crucially---is constrained and exhibits non-Euclidean geometry. 
    
    \item Smoothness of the value function function is established w.r.t.~a (non-Euclidean) norm that measures distance between policies in a manner that is independent of the size of the state space. 
    The choice of norm can be global (e.g., $\norm{\cdot}_{\infty, 1}$), or a local norm induced by the on-policy occupancy measure.

    \item By the policy gradient theorem and a standard uniform concentration argument,
    we have that a gradient step on the value function may be approximated through on-policy sampling.
    
\end{enumerate}
In more detail, consider that a standard template for first order optimization in policy space may be framed as iterating through minimization problems of the form:
\begin{align}\label{def:pi_grad_update}
    \pi^{k+1}
    \gets 
    \argmin_{\pi \in \Pi} 
        % \cbr{ \Phi_k(\pi) =
            \abr[b]{\nabla V(\pi^k), \pi} 
                + \frac1{\eta} \Div_{\Pi}\br[b]{\pi ,\pi^k}
        % }
        ,
\end{align}
where $\pi^1\in \Pi$ is a given initialization and $\eta>0$ a learning rate.
When $\Pi$ is a convex set and $\Div_{\Pi}$ is a distance-like function that is compatible with the geometry of the objective function $V$ --- e.g., informally, when $V$ is smooth w.r.t.~$\Div_{\Pi}$ --- some form of convergence may be established.
In what follows, by a direct application of the policy gradient theorem \citep{sutton1999policy} and the recently established local smoothness property \citep{sherman2025convergence}, we will demonstrate how policy learning may be reduced to first order optimization in the form of \cref{def:pi_grad_update}.
Let $\pi^k$ is the agent's policy on iteration $k$, and let $\mu^k, Q^k$ be it's occupancy measure and action-value function.
By the policy gradient theorem \citep{sutton1999policy}, we have that for any policy $\pi$:
\begin{align*}
    \E_{s \sim \mu^k}\sbr{
            H\abr{Q^{k}_{s}, \pi_s} 
        }
    =\abr[b]{\nabla V(\pi^k), \pi} 
    .
\end{align*}
Thus, using a proper sampling mechanism while interacting with the environment, we may obtain a dataset $\cD_k$ 
consisting of states $s\sim \mu^k \eqq \mu^{\pi^k}$, and unbiased action-value estimates $\widehat Q^k_{s}$.
This gives an empirical version of the linearization 
$\frac HN \sum_{s\in \cD_k}
            \abr[b]{\widehat Q^{k}_{s}, \pi_s} 
        $ that concentrates about the true gradient as $N$ grows.
Conveniently, given a distance-like function $\Div\colon \R^\cA \times \R^\cA \to \R$ over the action space, taking an expectation w.r.t.~the on-policy distribution gives a policy distance-like function 
$\Div_k(\pi, \pi') \eqq \E_{s\sim \mu^k}\Div \br[b]{\pi_s ,\pi'_s}$ w.r.t.~which the value function is smooth \citep{sherman2025convergence}. Importantly, in order to enjoy this smoothness, we must incorporate $\epsexpl$-greedy exploration. To that end, we define the exploratory verion of $\Pi$ as follows:
\begin{align}\label{def:Pi_expl}
    \Pi^\epsexpl \eqq \cbr{(1-\epsexpl)\pi + \epsexpl u \mid 
    \pi \in \Pi, u_{s,a} \equiv 1/A\;\forall s, a}
\end{align}
        
With this in mind, we consider on-policy algorithms, 
all of which hinge on optimizing  empirical surrogates to the full gradient objective function. Concretely, the update step in each algorithm is of the form:
\begin{align}\label{def:hat_phi_k}
        \pi^{k+1}
        \gets 
        \argmin_{\pi\in \Pi^{\epsexpl}}\cbr{
        \widehat \Phi_k(\pi) \eqq
        \frac H N \sum_{s \in \Dst_k} {
            \abr{\widehat Q^{k}_{s}, \pi_s} 
            + \frac1{\eta} \Div\br[b]{\pi_s,\pi_s^k}
        }}.
   \end{align}
   The preceding discussion implies that, when $N$ is sufficiently large, \cref{def:hat_phi_k} is an approximate version of \cref{def:pi_grad_update} with a distance measure $\cD_\Pi$ that is adapted to local-smoothness of the objective, and as such \cref{def:hat_phi_k} is an instance of smooth non-convex optimization in a non-Euclidean space.
   Primarily, the algorithms we consider in \cref{sec:algs} differ in the choice of $\Div$.
   % --- given an action norm $\norm{\cdot}_\circ$: 
   % \begin{itemize}[leftmargin=!]
   %     \item SDPO uses $\Div\br{\pi_s,\pi_s^k}=\frac12\norm{\pi_s - \pi_s^k}_\circ^2$.
   %     \item PMD uses 
   % $\Div\br{\pi_s,\pi_s^k}=\Breg_R(\pi_s, \pi_s^k)$ for some action regularizer $R$ that is typically $1$-strongly convex w.r.t.~the action norm $\norm{\cdot}_\circ$.
   %      \item CPI uses $\Div\br{\pi_s,\pi_s^k}=0$ to compute a greedy policy $\tilde \pi^{k+1}$. DA-CPI performs an additional step required to compute $\pi^{k+1} \approx (1-\eta_k)\pi^k + \eta_k \tilde \pi^{k+1}$. Local-smoothness of the value function plays an important role in this second step, as it too is approximated via on-policy sampling.
   % \end{itemize}
   % We note that the sample complexity guarantee of CPI subject to $\Pi$-completeness and coverage
   % Our contribution, with regards to CPI (aside from the improvement in iteration complexity), and further placing it in the context of a deeper connection, as a result that stems from a pure optimization perspective.
    For the formal definition of the optimization setup our reduction leads to, we refer the reader to \cref{sec:opt}.
    Finally, we note that while our algorithms SDPO, DA-CPI, and PMD operate over the $\epsexpl$-exploratory version of $\Pi$, their output policies may be transformed back into $\Pi$ with negligible loss of the objective, hence they are in fact proper agnostic learning algorithms.

\section{Policy learning algorithms}
\label{sec:algs}
% \label{sec:rl_algs}
In this section, we present our policy learning algorithms in their idealized form. Given the discussion from \cref{sec:framework_reduction}, sample complexity of each algorithm follows through the same algorithmic template and argument; we provide the full details in \cref{sec:sc}.

\subsection{Steepest Descent Policy Optimization}
In this section, we present our first algorithm, which we derive  from a generalization of gradient descent to non-Euclidean norms. Given a differentiable objective $f\colon \R^d \to \R$, an unconstrained, Euclidean gradient descent step can be written as:
\begin{align*}
    x^+ = x - \eta \nabla f(x)
    = \argmin_{y\in \R^d} \abr{\nabla f(x), y} + \frac{1}{2\eta}\norm{y - x}_2^2.
\end{align*}
When the objective $f$ is smooth w.r.t.~$\norm{\cdot}_2$ and the step size is chosen appropriately, it is guaranteed that the step decreases the objective value, which can by harnessed to obtain convergence to a stationary point in a non-convex setting.
A natural generalization of the gradient descent step to accommodate non-Euclidean geometries consists of simply replacing the proximity term $\norm{y - x}_2^2$ with any other norm $\norm{\cdot}$:
\begin{align*}
    x^+ = \argmin_{y\in \R^d} \abr{\nabla f(x), y} + \frac{1}{2\eta}\norm{y - x}^2.
\end{align*}
As in the Euclidean case, when $f$ is smooth w.r.t.~$\norm{\cdot}$ and the step size is chosen appropriately, the step decreases the objective value, which leads to convergence when applied iteratively over $K$ steps. A relatively straightforward argument can be employed to establish $O(1/\sqrt K)$ convergence to an approximate stationary point.
The work of \citet{kelner2014almost} appears to be the first to prove convergence of $O(1/K)$ assuming convexity of $f$, at least in this general form, as noted by the authors. In this work, we analyze for the first time the constrained version of the steepest descent method, which requires some care and extra machinery to cope simultaneously with constraints and the non-Euclidean nature of the updates. We obtain a $O(1/K)$ upper bound for VGD functions (a condition weaker than convexity) and $O(1/\sqrt K)$ for general non-convex functions. To the best of our knowledge, this method was not considered by any prior work in the constrained setting, for any class of objective functions.

\begin{algorithm}[tb]
   \caption{Steepest Descent Policy Optimization (SDPO)}
   \label{alg:sdpo}
\begin{algorithmic}
   \STATE {\bfseries Input:} 
   $K\geq 1, \eta > 0, \varepsilon > 0, \epsexpl >0, \Pi\in \Delta(\cA)^\cS$, and action norm $\norm{\cdot}_\circ\colon\R^\cA \to \R$.
   \STATE Initialize $\pi^1 \in \Pi^\epsexpl$
   \FOR{$k=1$ {\bfseries to} $K$}
   \STATE Update
   \begin{aligni*}
        \pi^{k+1}
        \gets 
        \argmineps_{\pi\in \Pi^\epsexpl}
        \E_{s\sim \mu^k} \sbr{
            H\abr{Q^{k}_{s}, \pi_s} 
            + \frac1{2\eta} \norm{\pi_s - \pi_s^k}_\circ^2
        }
   \end{aligni*}
   \ENDFOR
   \RETURN $\hat \pi \eqq \pi^{K+1}$
\end{algorithmic}
\end{algorithm}

The iteration complexity given by \cref{thm:sdpo} below, 
follows by our reduction explained in \cref{sec:framework_reduction} (excluding the probabilistic part). The $\varepsilon$ parameter passed to \cref{alg:sdpo} should be understood as the sum of the generalization error, originating in the noisy estimates of the objective, and optimization error, originating from the error of the optimization oracle.
\begin{theorem}\label{thm:sdpo}
    Let $\Pi$ be a convex policy class that satisfies $(\cvgd, \epsvgd)$-VGD w.r.t.~$\cM$.
    Suppose that SDPO (\cref{alg:sdpo}) is executed with the $L^1$ action norm $\norm{\cdot}_1$.
    Then, after $K$ iterations, with appropriately tuned $\eta$ and $\epsexpl$, the output of SDPO satisfies:
    \begin{align*}
        V(\hat \pi) - V^\star(\Pi)
        &= O\br{
        \frac{ \cvgd^2 A H^{4}}{ K^{2/3}}
        + \cvgd H^3 \sqrt A K^{1/6}\sqrt{\varepsilon}
        + \epsvgd}.
    \end{align*}
\end{theorem}
As mentioned in the introduction, SDPO improves upon the previous guarantees of PMD \citep{sherman2025convergence} in two regards; (i) dependence on the error $\varepsilon^{1/4} \to \varepsilon^{1/2}$, and (ii) dependence on $A$. Further, for the case of Euclidean action geometry, the SDPO analysis can be seen as a tighter analysis for PMD with Euclidean action-regularization case (note that the space the algorithms operate in is still non-Euclidean, only the action space is).

\subsection{Conservative Policy Iteration}
\label{sec:cpi}
In this section, we present our results relating to the CPI algorithm (\citealp{kakade2002approximately}; see also \citealp{agarwal2019reinforcement}).
We first consider CPI in its original form, presented in \cref{alg:cpi}. 
\citet{kakade2002approximately} established, that when the step sizes $\eta_k$ are chosen ``greedily'' so as to maximize the observed advantage gain, an $O(1/\varepsilon^2)$ iteration complexity follows, as long as the policy class is complete and the distribution mismatch is bounded (see our introduction and \cref{sec:comp_prior_art}). 
\begin{algorithm}[ht!]
    \caption{Conservative Policy Iteration (CPI; \citealp{kakade2002approximately})} 
    \label{alg:cpi}
	\begin{algorithmic}
	    \STATE \textbf{input:} Initial policy $\pi^1\in\Pi$, Error tolerance $\varepsilon>0$
           \FOR{$k=1,2, \ldots,$}
           \STATE Update
           \begin{aligni*}
                \tilde \pi^{k+1}
                \gets 
                \argmineps_{\pi\in \Pi}
                \E_{s\sim \mu^k}
                    \abr{H Q^{k}_{s}, \pi_s - \pi^k_s} 
           \end{aligni*}
			\STATE Set
                $\pi^{k+1}  = (1-\eta_k) \pi^k + \eta_k \tilde \pi^{k+1}$
            \ENDFOR
	\end{algorithmic}
\end{algorithm}

However, as it turns out, this is not the optimal step size choice. The following two observations imply that CPI is (an approximate version of) the Frank-Wolfe \citep{frank1956algorithm} algorithm applied in state-action space, with policies as the optimization variables and the value function as the objective. Indeed, we have that (i) $\E_{s\sim \mu^k}\abr{Q^{k}_{s}, \pi_s - \pi^k_s} = \tfrac1H \abr{\nabla V(\pi^k), \pi - \pi^k}$, and further that (ii) the value function is $(2 H^3)$-smooth w.r.t.~the $\norm{\cdot}_{\infty,1}$ norm --- a property which we prove in \cref{sec:proofs_cpi_dacpi}. Given (i) and (ii), it follows that convergence of CPI may be established through a standard FW analysis, where indeed, greedily choosing the step sizes gives $O(1/\varepsilon^2)$ iteration complexity, while the choice of $\eta_k\approx\frac1k$ gives $O(1/\varepsilon)$ iteration complexity.
\begin{theorem}\label{thm:cpi}
	Let $\Pi$ be a policy class that satisfies $(\cvgd, \epsvgd)$-VGD w.r.t.~$\cM$.
    Suppose that CPI (\cref{alg:cpi}) is executed with the step size choices $\eta_k=\frac{2\cvgd}{k+2}$ for $k=1, \ldots, K$. Then, we have the guarantee that:
    \begin{align*}
    	V(\pi^K) - V^\star(\Pi)
    	\leq \frac{8(2\cvgd^2 + 1)H^3}{K} + 2\cvgd\varepsilon
        + \epsvgd
    \end{align*}
\end{theorem}
As mentioned, completeness and coverage imply VGD (see \cref{sec:comp_prior_art} for the full details), hence \cref{thm:cpi} is a strict improvement over the classically known $O(1/\sqrt K)$ guarantees established in \citet{kakade2002approximately}. Also in \cref{sec:comp_prior_art}, we show our bounds subject to VGD subsume those with the completeness error $\cE(\Pi)$ \citep{scherrer2014local}. Finally,  
we note that in this version of the algorithm, convexity of $\Pi$ is not required (see \citealp{scherrer2014local} for additional discussion), nor does our analysis require it.

\paragraph{Doubly Approximate CPI: An actor-oracle efficient algorithm.}
In the function approximation setup, the convex combination step \emph{cannot} be computed as is, and therefore CPI turns out as actor-oracle inefficient. Indeed, the algorihtm requires keeping all policies $\pi^1, \ldots, \pi^K$ throughout execution. A possible remedy for this issue is to approximate the convex combination step with a second oracle invocation, thereby allowing to dispose of actors computed in previous rounds.
However, a-priori, extending the analysis is far from immediate, as the convex combination step needs to be approximated through samples and it is unclear how to control the propagation of errors in the analysis.
Fortunately, the local-smoothness framework together with a local-norm based analysis of FW (which in itself is straightforward) allows to control on-policy estimation errors and arrive at a convergence rate upper bound.

\begin{algorithm}[ht!]
    \caption{Doubly-Approximate CPI (DA-CPI)} 
    \label{alg:da_cpi}
	\begin{algorithmic}
	    \STATE \textbf{input:} $
	    	\eta_1, \ldots, \eta_K > 0
	    $; $\varepsilon>0, \epsexpl>0$, action norm $\norm{\cdot}_\circ$.
            \FOR{$k=1, \ldots, K$}
            \STATE Update
            \begin{aligni*}
                \tilde \pi^{k+1}
                \gets 
                \argmineps_{\pi\in \Pi^\epsexpl}
                \E_{s \sim \mu^k} {
                    \abr{H Q^{k}_{s}, \pi_s - \pi^k_s} 
                }
           \end{aligni*}
           \STATE Update
                $\pi^{k+1} \gets 
                \argmineps_{\pi\in \Pi^\epsexpl} 
                    \E_{s \sim \mu^k} 
                    \norm{\pi_s - ((1-\eta_k) \pi^k_s + \eta_k \tilde \pi^{k+1}_s)}_\circ^2
                $
            \ENDFOR
            \RETURN $\hat \pi = \pi^{K+1}$
	\end{algorithmic}
\end{algorithm}

\begin{theorem}\label{thm:da_cpi}
    Let $\Pi$ be a convex policy class that satisfies $(\cvgd, \epsvgd)$-VGD w.r.t.~$\cM$.
    Suppose that DA-CPI (\cref{alg:da_cpi}) is executed with the $L^1$ action norm $\norm{\cdot}_1$.
    Then, for an appropriate setting of $\eta_1, \ldots, \eta_K$ and $\epsexpl$, we have that the DA-CPI output satisfies:
    \begin{align*}
        V(\hat \pi) - V^\star(\Pi)
        &= O\br{\cvgd^2 AH^3\br{
	   \frac{1}{K^{2/3}} 
            + \varepsilon^{1/3}
    	+ \varepsilon^{2/3}  K^{2/3}}
        }
    .
    \end{align*}
\end{theorem}

\subsection{Policy Mirror Descent}
Convergence of PMD in the optimization setting similar to the one we consider here was recently established in 
\cite{sherman2025convergence}. With moderate additional work, we obtain the following iteration complexity upper bound that may be directly translated to a sample complexity guarantee. Since there are no changes in the algorithm, we defer its presentation to \cref{sec:proofs_pmd}.
\begin{theorem}\label{thm:pmd}
    Let $\Pi$ be a convex policy class that satisfies $(\cvgd, \epsvgd)$-VGD w.r.t.~$\cM$.
    Suppose that PMD (see \cref{alg:pmd} in \cref{sec:proofs_pmd}) is 
    executed with the $L^2$ action regularizer.
    Then, with an appropriate tuning of  $\eta, \epsexpl$, we have that the output of PMD satisfies:
    \begin{align*}
        V(\hat \pi) - V^\star(\Pi)
        &= O\br{
            \frac{\cvgd^2 A^{3/2} H^3 }{ K^{2/3}}
            + \br{\cvgd + H^2 A K^{1/6}}\varepsilon^{1/4}
            }.
    \end{align*}
\end{theorem}
\cref{thm:pmd} above, along with our reduction (see \cref{sec:framework_reduction} and \cref{sec:sc}), to our best knowledge lead to the first sample complexity upper bounds for PMD in the agnostic setting, which are completely independent of the policy class parametrization.

\section{Experimental evaluation of the VGD condition}
\label{sec:vgd_experiments}
In this section, we present proof-of-concept experiments (\cref{fig:vgd_exps}), demonstrating the empirically observed parameters of $L^2$-SDPO (equivalently, $L^2$-PMD) executed in four environments: Cartpole-v1 and Acrobot-v1  
\citep{brockman2016openai}, and SpaceInvaders-MinAtar and Breakout-MinAtar \citep{young19minatar}. Code was written on top of the Gymnax framework \citep{gymnax2022github}, and parts of it were based off purejaxrl \citep{lu2022discovered}.

\begin{figure}[ht]
    \centering
    \includegraphics[width=0.95\textwidth]{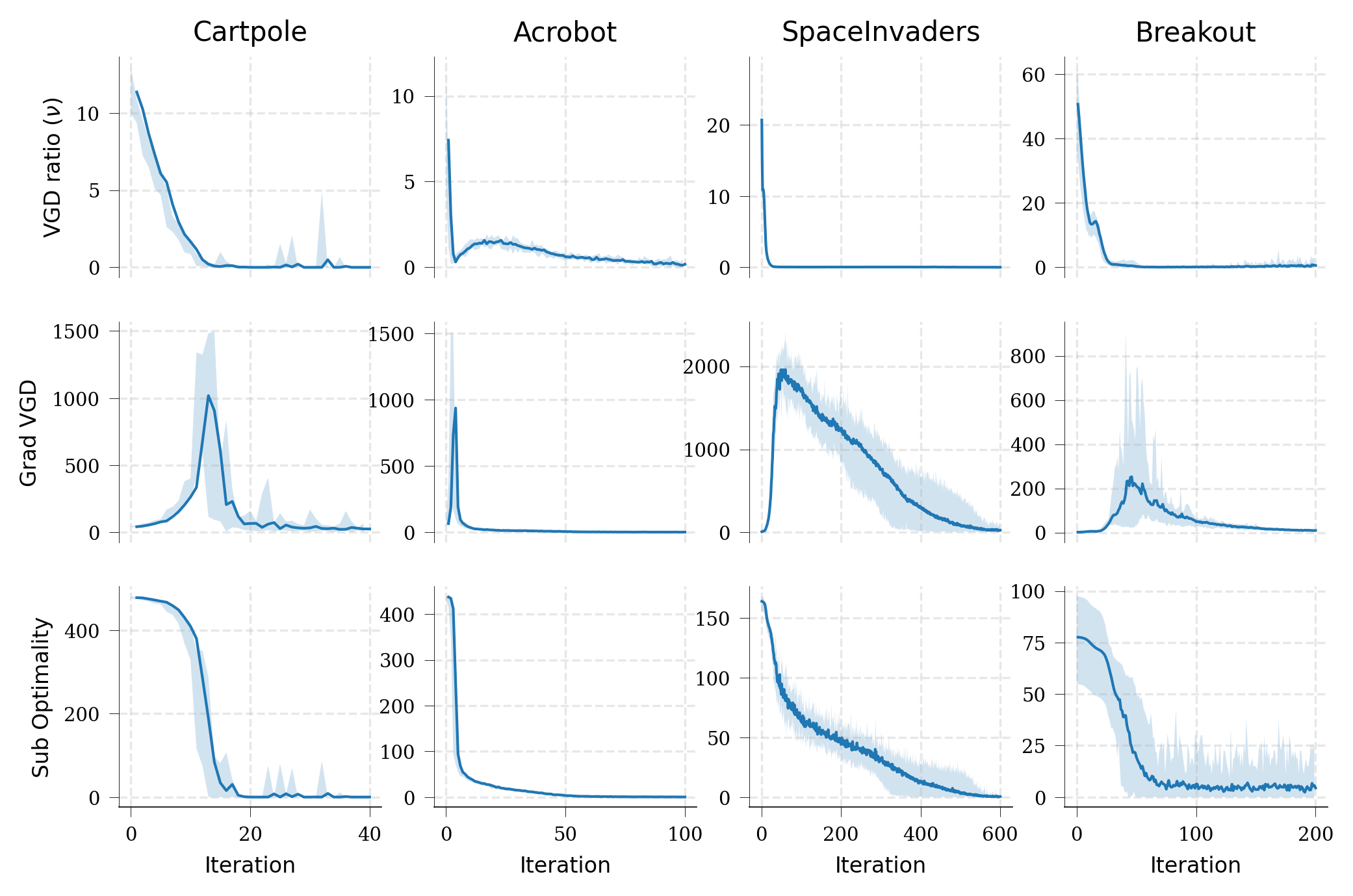}
    \caption{
        Training plots for all environments. In each experiment, a single set of minimally tuned hyper-parameters was run with 10 difference seeds. Error bars indicate maximum and minimum values. 
        \textbf{VGD Ratio ($\nu$):} An estimate of the $\cvgd$ parameter observed at each iteration $k\in [K]$ of the algorithm.
        \textbf{Grad VGD:} An estimate of $\max_{\tilde \pi\in \Pi}\abr{\nabla V(\pi^k), \pi^k - \tilde \pi}$.
        \textbf{Sub Optimality:} 
        Sub optimality of iteration $k$ w.r.t.~the minimum value the algorithm converged to. This should be interpreted as $\smash{V(\pi^k) - \br[s]{V^\star(\Pi) + \epsvgd}}$.
    }
    \label{fig:vgd_exps}
\end{figure}
As can be seen in \cref{fig:vgd_exps}, the estimates of the $\cvgd$-VGD parameter coefficient remain moderate throughout execution, and in fact decrease to around $1$ or below as the algorithm approaches convergence.

We maintain two neural network models in the algorithm implementation, one for the original SDPO actor model, the other (a ``VGD actor'') we use to estimate the VGD parameter.
In each iteration, we compute the SDPO step to obtain $\pi^{k+1}$, and 
in addition optimize the advantage function with the VGD actor in order to estimate $\max_{\tilde \pi\in \Pi}\abr{\nabla V(\pi^k), \pi^k - \tilde \pi}$. 
% We employ an extensive sampling procedure to obtain a good estimate of the linearization of $V$ at $\pi^k$.
%
Since the algorithm may converge to a local optimal (which the $\epsvgd$ parameter accounts for) we take the minimum value of each execution as the error floor $V^\star(\Pi) + \epsvgd$, and define the sub optimality as $V(\pi^k)- \br{V^\star(\Pi) + \epsvgd}$. We then report an estimate of $\cvgd$ by computing:
\begin{align}\label{eq:vgd_nu_exps}
    \frac{V(\pi^k) - \br{V^\star(\Pi) + \epsvgd}}
    {\max_{\tilde \pi\in \Pi}\abr{\nabla V(\pi^k), \pi^k - \tilde \pi}}
    = \nu_k .
\end{align}
In practice, for Cartpole and Acrobot the algorithm always converged to the global minimum; local optima only plays a role in the two more challenging environments. Given we believe the error floor, our reported $\cvgd_k$ are overestimates of the true parameter at $\pi^k$, as we find an actual policy $\tilde \pi$ that gives an upper bound on the LHS of \cref{eq:vgd_nu_exps}.
We provide further details on the experimental setup in \cref{sec:vgd_experiments_details}.

\section*{Acknowledgements}
The authors would like to thank Tal Lancewicki and Alon Cohen for fruitful discussions and for their participation in the initial ideation process of this work.
This project has received funding from the European Research Council (ERC) under the European Union’s Horizon 2020 research and innovation program (grant agreements No.~101078075; 882396).
Views and opinions expressed are however those of the author(s) only and do not necessarily reflect those of the European Union or the European Research Council. Neither the European Union nor the granting authority can be held responsible for them.
This work received additional support from the Israel Science Foundation (ISF, grant numbers 3174/23; 1357/24), and a grant from the Tel Aviv University Center for AI and Data Science (TAD). This work was partially supported by the Deutsch Foundation.

% \newpage
% \input{neurips25_cl}
\bibliographystyle{abbrvnat}
\bibliography{main}

% \newpage
\appendix

\section{Comparison with prior art}
\label{sec:comp_prior_art}
In this section, we provide additional details regarding related work and relation between different assumptions mentioned in the introduction.
We begin with formal definitions for $\Pi$-\emph{completeness} and \emph{coverage}.
Completeness is also sometimes referred to as closure \citep{bhandari2024global,sherman2025convergence}, and coverage is a general term we use here to refer to bounded distribution-mismatch coefficient \citep{agarwal2019reinforcement}.
\begin{definition}[Completness]
	We say a policy class $\Pi$ is complete if for any $\pi\in \Pi$, there exists a policy $\pi^+\in \Pi$ such that for all $s\in \cS$, $\pi^+_s= e_a$, where $a\in\argmax_{a\in \cA}Q^\pi_{s, a}$. In words, $\Pi$ contains a policy $\pi^+$ that acts greedily w.r.t.~the $Q$-function of $\pi$.
\end{definition}
\begin{definition}[Coverage / counded distribution-mismatch coefficient]
	We say the environment satisfies coverage if 
	\begin{align}
		D_\infty \eqq \norm{\frac{\mu^\star}{\rho_0}}_\infty = \max_{s\in \cS}\frac{\mu^\star(s)}{\rho_0(s)} < \infty,
	\end{align}
	where $\mu^\star = \mu^{\pi^\star}$ and $\pi^\star=\argmin_{\pi\in \Delta(\cA)^\cS} V(\pi)$.
\end{definition}
We note that while it makes sense to consider coverage subject to the best-in-class policy, it is usually considered in conjunction with completeness and therefore in the realizable setting.
The following is a notion of approximate completeness that has appeared in e.g., \citet{scherrer2014local, bhandari2024global}.
\begin{definition}[Approximate completeness / approximate closure]
	The completeness error of a policy class $\Pi$ is defined by:
	\begin{align}
	\cE(\Pi)
	&\eqq 
	\max_{\pi\in \Pi}
	\cbr{
	\max_{\pi^+ \in \Delta(\cA)^\cS}\E_{s \sim \mu^\pi}\abr{Q^\pi_s, \pi_s - \pi^+_s}
	- 
	\max_{\tilde \pi \in \Pi}\E_{s \sim \mu^\pi}\abr{Q^\pi_s, \pi_s - \tilde \pi_s}
	}.
	\end{align}
\end{definition}
We note that as long as we consider strictly stochastic policies with occupancy measures that have full support, $\cE(\Pi)=0$ implies $\Pi$ is complete.
We have the following relation between (approximate) completeness, coverage, and the VGD condition. For proof see \citet{sherman2025convergence}, which in itself is based on \citep{agarwal2021theory,bhandari2024global}.
\begin{lemma}
	Let $\Pi$ be a policy class and suppose coverage holds, i.e., $D_\infty<\infty$. Then $\Pi$ satisfies $(\cvgd, \epsvgd)$-VGD with $\cvgd = H D_\infty$ and $\epsvgd = \cE(\Pi) H^2 D_\infty$. In particular, if $\Pi$ is complete it satisfies $(H D_\infty, 0)$-VGD.
\end{lemma}
We note that the VGD error floor in the above lemma is identical to the error floor in the convergence guarantees of CPI subject to approximate completeness (\citealp{scherrer2014local}, see also \citealp{agarwal2019reinforcement}).
(Recall that $H\eqq\frac{1}{1-\gamma}$ denotes the effective horizon.)
\paragraph{Comparison with algorithms of interest.}
In what follows, we discuss some features of the algorithms listed in \cref{tab:po_algs_comparison}.
%We note that sample complexity results obtained by prior works, while they might require additional assumptions such as particular parameterizations, they give faster rates than those we obtain here. Further, they do not depend on the log covering number of the policy class, as ours do.
\begin{itemize}
	\item \textbf{CPI} \citep{kakade2002approximately}. As discussed in the introduction and \cref{sec:cpi}, CPI is not actor-oracle efficient. The bound with the completeness-error error floor does not capture non-realizable setups where best-in-class convergence is possible, subject to the VGD condition (see \citealp{sherman2025convergence} for a simple example). 
	\item  \textbf{Log-linear NPG} \citep{agarwal2021theory}. In the log-linear NPG setup the policy class is parametrized by softmax-over-linear functions, where linear is w.r.t.~given state-action features. When the linear function class can represent  all action-value functions with zero-error, the policy class is essentially complete (formally, it is approximately complete with any desired non-zero error.). $\varepsilon_{\rm approx}$ stands for a bound on the least-squares error of the  approximation.
	The relative condition number $\kappa$ is defined in Assumption 6.2.
	Theorem 20 in their work gives a sample complexity upper bound for Q-NPG, which is a version of PMD with linear function approximation suited for the approximately-complete $\Pi$ learning setup.
	\item \textbf{Log-linear NPG} \citep{yuan2022general}. $\varepsilon_{\rm approx}$ is as in the work of \citet{agarwal2021theory}, $C_v$ is a concentrability coefficient that is generally stronger than $D_\infty$, and relates to how well are occupancy measures of policies chosen by the algorithm supported on the optimal policy. Their sample complexity for Q-NPG requires an additional relative condition number assumption similar to $\kappa$ that is not present in the iteration complexity upper bound.
	\item  \textbf{PMD} \citep{alfano2022linear}. Here, the policy class is composed from a general parametrization of the PMD dual variables, and an exact mirror-and-project function. $\varepsilon_{\rm pmdc}$ quantifies the least-square error in approximating the dual variables of policies chosen by the exact PMD step. Roughly speaking, this is a generalization of the $\varepsilon_{\rm approx}$ of \citet{agarwal2019reinforcement} to the more general parametrization setup they consider. A sample complexity result is given for the specific case of a shallow neural network parametrization.
	\item  \textbf{PMD} \citep{sherman2025convergence}. Here, the policy class parametrization is completely general. Sample complexity follows by arguments we present in \cref{sec:sc}. The analysis of \citet{sherman2025convergence} applies for non-convex policy classes, but only subject to PMD completeness (i.e., with dependence on $\varepsilon_{\rm pmdc}$).
\end{itemize}

% \subsection{Additional Related work}
% \label{sec:additional_related_work}

\section{Additional preliminaries}
\paragraph{Discounted MDPs.}
A discounted Markov Decision Process (MDP; \cite{puterman1994markov}) $\cM$ is defined by a tuple
	$\cM = (\cS, \cA, \P, r, \gamma, \rho_0)$.
For notational convenience, for $s,a\in \cS\times\cA$ we let $\P_{s, a} \eqq \P(\cdot \mid s, a) \in \Delta(\cS)$ denote the next state probability measure.
We denote the \emph{value} of $\pi$ when starting from a state $s\in \cS$ by $V_s(\pi)$:
\begin{align*}
    V_s(\pi) \eqq \E\sbr{\sum_{t=0}^\infty \gamma^t r(s_t, a_t) \mid s_0 = s, \pi},
\end{align*}
and more generally for any $\rho\in \Delta(\cS)$,
$V_\rho(\pi) \eqq \E_{s \sim \rho}V_{s}(\pi)$. When the subscript is omitted, $V(\pi)$ denotes value of $\pi$ when starting from the initial state distribution $\rho_0$:
\begin{align*}
    V(\pi) \eqq V_{\rho_0}(\pi) =  \E\sbr{\sum_{t=0}^\infty \gamma^t r(s_t, a_t) \mid s_0 \sim \rho_0, \pi}.
\end{align*}
For any state action pair $s,a\in \cS\times \cA$, the action-value function of $\pi$, or $Q$-function, measures the value of $\pi$ when starting from $s$, taking action $a$, and then following $\pi$ for the reset of the interaction:
\begin{align*}
    Q^\pi_{s, a} \eqq \E\sbr{\sum_{t=0}^\infty \gamma^t r(s_t, a_t) \mid s_0 = s, a_0=a, \pi}
\end{align*}
We further denote the discounted state-occupancy measure of $\pi$ induced by any start state distribution $\rho\in \Delta(\cS)$ by $\mu^\pi_\rho$:
\begin{align*}
    \mu^\pi_{\rho}(s) \eqq 
    \br{1-\gamma}\sum_{t=0}^\infty \gamma^t \Pr(s_t = s \mid s_0 \sim \rho, \pi).
\end{align*}
It is easily verified that $\mu^\pi\in \Delta(\cS)$ is indeed a state probability measure.
In the sake of brevity, we take the MDP true start state distribution $\rho_0$ as the default in case one is not specified:
\begin{align}\label{def:occmeas_rho0}
    \mu^\pi \eqq \mu^\pi_{\rho_0}.
\end{align}

\paragraph{Problem setup.}
We consider an infinite horizon environment interaction model (see Protocol \ref{prot:infinite_horizon_rollout}) similar to that of, e.g., \citet{agarwal2021theory}.

\begin{protocol}[!ht]
    \caption{Infinite horizon environment rollout}
    \label{prot:infinite_horizon_rollout}
	\begin{algorithmic}
            \STATE Environment resets $s_0 \sim \rho_0$
            \FOR{$t = 0, \ldots:$}
                \STATE Agent observes $s_t$,
                chooses $a_t$, incurs $r(s_t, a_t)$ 
                
                \STATE Environment transitions $s_{t+1} \sim \P(\cdot|s_t, a_t)$
                \STATE Agent decides whether to terminate episode
            \ENDFOR
	\end{algorithmic}
\end{protocol}

\paragraph{Additional notation.}
For a given set $\cX$, we let $\cN(\epsilon, \cX, \norm{\cdot})$ denote the $\epsilon$-\emph{covering number} of $\cX$ w.r.t.~the norm $\norm{\cdot}$. The convex closure of $\cX$ is denoted by ${\rm conv}(\cX)$. We recall our definition for an $\epsexpl$-greedy exploratory version of a policy class $\Pi$ (see \cref{def:Pi_expl}):
\begin{align*}
    \Pi^\epsexpl \eqq \cbr{(1-\epsexpl)\pi + \epsexpl u \mid 
    \pi \in \Pi, u_{s,a} \equiv 1/A\;\forall s, a}
\end{align*}
We conclude by recalling the following notations used throughout:
\begin{align*}
    \mu^k \eqq \mu^{\pi^k}, \quad 
    Q^k \eqq Q^{\pi^k}, \quad 
    S \eqq |\cS|, \quad
    A \eqq |\cA|, \quad
    H \eqq \frac{1}{1-\gamma}.
\end{align*}

\section{First Order Methods for non-Euclidean Optimization}
\label{sec:opt}
In this section, we consider the smooth non-convex optimization problem:
\begin{align}\label{def:opt_main}
    \min_{x\in \cX} f(x),
\end{align}
where $f\colon \cX \to \R$ and 
$\cX \subset \R^d$ is a compact convex set. Our reduction detailed in \cref{sec:framework_reduction} leads to the optimization setup described next. All algorithms presented in \cref{sec:algs} are instances of the algorithms we analyze in this section.
Before introducing the problem setup, we give the following definitions.
\begin{definition}[Variational Gradient Dominance]\label{def:vgd}
    We say $f\colon \cX \to \R$ satisfies the variational gradient dominance condition with parameters $(\cvgd, \epsvgd)$, or that $f$ is $(\cvgd, \epsvgd)$-VGD, if
    here exist constants $\cvgd, \epsvgd>0$, such that for any $x \in \cX$, it holds that:
    \begin{align*}
        f(x) - \argmin_{x^\star\in \cX}f(x^\star)
        \leq 
        \cvgd \max_{\tilde x \in \cX}\abr{\nabla f(x), x - \tilde x} + \epsvgd.
    \end{align*}
\end{definition}

% \begin{definition}[Smoothness]\label{def:smoothness}
%     We say $f\colon \cX \to \R$ is
%     $\beta$-smooth w.r.t.~$\norm{\cdot}$ if for all $x, y\in \dom f$:
%     \begin{align*}
%     \av{f(y) - f(x) - \abr{\nabla f(x), y - x}} \leq 
%     \frac\beta2\norm{y-x}^2 .
%     \end{align*}
% \end{definition}

\begin{definition}[Local Norm]\label{def:local_norm}
    We define a \emph{local} norm over a set $\cX\subseteq \R^d$ by a mapping $x \mapsto \norm{\cdot}_x$ such that $\norm{\cdot}_x$ is a norm for all $x\in \cX$.
    We may denote a local norm by $\norm{\cdot}_{(\cdot)}$ or by $x \mapsto \norm{\cdot}_x$.
\end{definition}

\begin{definition}[Local Smoothness]\label{def:local_smoothness}
    We say $f\colon \cX \to \R$ is
    $\beta$-\emph{locally} smooth w.r.t.~a local norm $x \mapsto \norm{\cdot}_x$ if for all $x, y\in \cX$:
    \begin{align*}
    \av{f(y) - f(x) - \abr{\nabla f(x), y - x}} \leq 
    \frac\beta2\norm{y-x}_x^2 .
    \end{align*}
\end{definition}

\begin{definition}[Local Lipschitz continuity]\label{def:local_Lipschitz}
    We say $f\colon \cX \to \R$ is
    $M$-\emph{locally} Lipschitz w.r.t.~a local norm $x\mapsto \norm{\cdot}_x$ if for all $x\in \cX$:
    \begin{align*}
    \norm{\nabla f(x)}_x^* \leq M.
    \end{align*}
\end{definition}

\paragraph{Problem Setup.}
We consider the problem of computing an approximate global minimum of the objective \cref{def:opt_main},
under the following conditions.
\begin{assumption}\label{assm:opt_setup}
    The local norm $x \mapsto \norm{\cdot}_x$, decision set $\cX$, and objective $f$ satisfy:
    \begin{itemize}
    \item $f$ is $\beta$-locally smooth and $M$-locally Lipschitz w.r.t.~$x \mapsto \norm{\cdot}_{x}$ for $M\geq 1$,
    \item $f$ is $(\cvgd, \epsvgd)$-VGD over $\cX$ with $\nu \geq 1$,
    \item $f$ is bounded from below; $f^\star=\min_{x} f(x) > -\infty$,
    \item $\cX$ has a bounded diameter w.r.t.~$x\mapsto \norm{\cdot}$; $D \geq \max\{1, \max_{x,y,z\in \cX}\norm{z - y}_x\}$.
\end{itemize}
The assumption that $D, M, \nu \geq 1$ is solely for simplicity of presentation; if for example the objective $f$ is $M$-Lipschitz for $M<1$, our results still hold with $M \to \max\{1, M\}$.
\subsection{Constrained steepest descent method}
In this section, we consider the constrained steepest descent method; given an initialization $x_1\in \cX$, step size $\eta>0$, and step error $\epserr>0$:
    \begin{align}\label{def:csd}
        \forall k=1, \ldots, K,\;
        x_{k+1} \in \argmineps_{x\in \cX}
        \cbr{\abr{\nabla f(x_k), x} + \frac1{2\eta}\norm{x - x_k}^2_{x_k}}.
    \end{align}
For this method, we prove the following theorem.
\begin{theorem}\label{thm:csd}
Under \cref{assm:opt_setup},
    the constrained steepest descent method \cref{def:csd} guarantees, as long as $\eta\leq 1/\beta$:
    \begin{align*}
        f(x_{K+1}) - f^\star
        \leq
        \frac{8\br{ \cvgd M D }^{2}}{\eta K}
        + \frac{4  \cvgd M D }{\sqrt \eta}\sqrt{\epserr}
        + \epsvgd
    \end{align*}
\end{theorem}
\end{assumption}
To prove \cref{thm:csd}, we will first establish an analysis framework that connects the algorithm with a notion of constrained steepest descent magnitude.
As a general note, the fact that the norm used in \cref{def:csd} to compute the steepest descent direction is local makes only a syntactic difference in the analysis. Wherever it is convenient we make a claim about a general norm (which may be local and depend on some point), like in \cref{lem:csd_gt_step} below.

Let $\eta > 0$ be a fixed step size.
Given $x\in \cX$, we define the set of potential gradient mappings from $x$ by:
\begin{align}\label{def:csd_gm_set}
	\cG_x \eqq \cbr{\frac1\eta(x - y) \mid y \in \cX},
\end{align}
and the steepest descent magnitude by:
\begin{align}\label{def:csd_dm}
\Des_x \eqq \max_{g \in \cG_x} \cbr{ \abr{\nabla f(x), g} 
    - \frac{1}{2}\norm{g}_x^2}.
\end{align}

\begin{lemma}\label{lem:csd_gt_step}
     For any norm $\norm{\cdot}$, It holds that
    \begin{align*}
    x^+ \in \argmineps_{y\in \cX} \cbr{ \abr{\nabla f(x),y} 
		+ \frac{1}{2\eta}\norm{x - y}^2}&
       \\ \iff
\frac1\eta(x - x^+) \in &\argmaxeps[(\epserr/\eta)]_{g \in \cG_x} \cbr{ \abr{\nabla f(x), g} 
		- \frac{1}{2}\norm{g}^2}.
\end{align*}
\end{lemma}
\begin{proof}
    Observe:
\begin{align*}
\argmineps_{y\in \cX} \cbr{ \abr{\nabla f(x),y} 
		+ \frac{1}{2\eta}\norm{x - y}^2}
&=
\argmineps_{y\in \cX} \cbr{ \abr{\nabla f(x),y - x} 
		+ \frac{1}{2\eta}\norm{y - x}^2}
\\ &=
\argmineps[(\epserr/\eta)]_{y\in \cX} \cbr{ 
    \abr{\nabla f(x),\frac1\eta\br{y - x}} 
		+ \frac{1}{2}\norm{\frac1\eta(y - x)}^2}.
\end{align*}
Further, for any $y \in \cX$, letting $g =\frac1\eta(x - y)$, we have
\begin{align*}
\abr{\nabla f(x),\frac1\eta\br{x - y}} 
		+ \frac{1}{2}\norm{\frac1\eta(x - y)}^2
&=
\abr{\nabla f(x), - g} 
		+ \frac{1}{2}\norm{g}^2
.
\end{align*}
Hence, the gradient mapping $y \mapsto \frac1\eta(x - y)$ is a bijection between $\cX$ and $\cG_x$ that 
gives the same value for the LHS and RHS objectives in the above display. Thus, 
\begin{align*}
x^+ \in \argmineps_{y\in \cX} \cbr{ \abr{\nabla f(x), y} 
		+ \frac{1}{2\eta}\norm{x - y}^2}&
\\
\iff
\frac1\eta(x - x^+) \in &\argmineps[(\epserr/\eta)]_{g \in \cG_x} \cbr{ \abr{\nabla f(x), -g} 
		+ \frac{1}{2}\norm{g}^2}
.
\end{align*}
Finally, 
\begin{align*}
    \argmineps_{g \in \cG_x} \cbr{ \abr{\nabla f(x), -g} 
		+ \frac{1}{2}\norm{g}^2}
    =
    \argmaxeps[(\epserr/\eta)]_{g \in \cG_x} \cbr{ \abr{\nabla f(x), g} 
		- \frac{1}{2}\norm{g}^2},
\end{align*}
which completes the proof.
\end{proof}

\begin{lemma}\label{lem:csd_to_stationarity}
    For all $x\in \cX, \eta > 0$, we have:
$$
	\max_{y\in \cX}\abr{\nabla f(x), x - y}
	\leq 2\max\cbr{D, 1}\max\cbr{\Des_x, \sqrt \Des_x}.
$$
\end{lemma}
\begin{proof}
    Let $y\in\cX$, and note that
\begin{align*}
	\abr{\nabla f(x), x - y}
	=
	\eta  \abr{\nabla f(x), \frac1\eta(x - y)},
\end{align*}
hence
\begin{align}\label{eq:csd_stationarity_to_Dx}
	\max_{y\in \cX}\abr{\nabla f(x), x - y}
	= \eta \max_{g \in \cG_x}\abr{\nabla f(x), g}.
\end{align}
In what follows, we consider the set of gradient mappings restricted to direction $u$ and the corresponding descent quantity in direction $u$:
\begin{align*}
    \cG_x(u) \eqq \cG_x \cap \{\alpha u: \alpha\geq0\};
    \quad
    \Des_x(u) \eqq \max_{g\in \cG_x(u)}
    \abr{\nabla f(x), g} - \frac12\norm{g}_x^2.
\end{align*}
By \cref{lem:sd_to_grad}, we have:
\begin{align*}
    \begin{cases}
    \Des_x(u) = \frac12 \abr{\nabla f(x), u}^2
    & \abr{\nabla f(x), u} u\in \cG_x,
    \\
        \Des_x(u) \geq \frac12\max_{g \in \cG_x(u)}
  \abr{\nabla f(x), g}
    & \abr{\nabla f(x),  u} u\notin \cG_x.
    \end{cases}
\end{align*}
Now, since
\begin{align*}
    \Des_x
    = 
    \max_{g\in \cG_x}\cbr{\abr{\nabla f(x), g} - \frac12\norm{g}_x^2}
    =
    \max_{\norm{u}_x=1}\Des_x(u)
    ,
\end{align*}
it follows that:
\begin{align*}
    \max_{g \in \cG_x}\abr{\nabla f(x), g}
    &=
    \max_{u: \norm{u}_x=1}\max_{g \in \cG_x(u)}\abr{\nabla f(x), g}
    \\
    &\leq 
    \max_{u: \norm{u}_x=1}
    \max\cbr{2 \Des_x(u), \frac D\eta\sqrt{2 \Des_x(u)}}
    \\
    &=
    \max\cbr{2 \Des_x, \frac D\eta\sqrt{2 \Des_x}}
    \\
    &\leq
    \frac{2}{\eta}\max\{D, 1\}\max\cbr{\Des_x, \sqrt{\Des_x}}
    .
\end{align*}
Combining the above with \cref{eq:csd_stationarity_to_Dx}, the result follows.
\end{proof}

\begin{lemma}\label{lem:sd_to_grad}
    Let $x\in \cX$,  $u\in \R^d, \norm{u}_x=1$, and define:
\begin{align*}
    \cG_x(u) \eqq \cG_x \cap \{\alpha u: \alpha\geq0\};
    \quad
    \Des_x(u) \eqq \max_{g\in \cG_x(u)}
    \abr{\nabla f(x), g} - \frac12\norm{g}_x^2.
\end{align*}
Then,
\begin{align*}
    \begin{cases}
    \Des_x(u) = \frac12 \abr{\nabla f(x), u}^2
    & \abr{\nabla f(x), u} u\in \cG_x,
    \\
        \Des_x(u) \geq \frac12\max_{g \in \cG_x(u)}
  \abr{\nabla f(x), g}
    & \abr{\nabla f(x),  u} u\notin \cG_x.
    \end{cases}
\end{align*}

\end{lemma}
\begin{proof}
We have, by definition of $\Des_x(u)$:
\begin{align*}
  \Des_x(u) =
  \max_{\alpha\geq0:\alpha u\in \cG_x}
  \cbr{\abr{\nabla f(x), \alpha u}
  - \frac12\norm{\alpha u}_x^2},
\end{align*}
and we note that without constraining $\alpha u\in \cG_x$, we have
\begin{align*}
    \argmax_{\alpha\geq0}
  \cbr{\abr{\nabla f(x), \alpha u}
  - \frac12\norm{\alpha u}_x^2}
  =
  \argmax_{\alpha\geq0}
  \cbr{ \alpha \abr{\nabla f(x), u}
  - \frac{\alpha^2}{2}}
  = 
  \abr{\nabla f(x), u}.
\end{align*}
Now, set 
$A \eqq \{\alpha : \alpha u \in \cG_x\}$, and 
$\alpha_0 \eqq \sup\{\alpha : \alpha u \in \cG_x\}$.
Observe that since $\cG_x$ is closed and convex, we have $\alpha_0=\infty \implies A=[0, \infty)$, and $\alpha_0 < \infty \implies A=[0, \alpha_0]$.
Proceeding, we now consider the two cases from the lemma statement. 
\paragraph{Assume first that $\abr{\nabla f(x), u} u\in \cG_x$.} Then $\alpha_0 \geq \abr{\nabla f(x), u}$, and by since $\alpha = \abr{\nabla f(x), u} \in A$ minimizes the unconstrained problem, it also minimizes the constrained one, hence
\begin{align*}
  \Des_x(u) =
  \frac12\abr{\nabla f(x),  u}^2,
\end{align*}
as required.
\paragraph{Assume now that $\abr{\nabla f(x), u} u\notin \cG_x$.}Then $\alpha_0 < \abr{\nabla f(x), u}$, and since
\begin{align*}
  \alpha \mapsto \cbr{\abr{\nabla f(x), \alpha u}
  - \frac12\norm{\alpha u}_x^2
  =
  \alpha \abr{\nabla f(x), u}
  - \frac{\alpha^2}{2}}
\end{align*}
is monotonically increasing for $\alpha\in[0, \abr{\nabla f(x), \alpha u}]$, we obtain
\begin{align*}
  \Des_x(u) =
  \abr{\nabla f(x), \alpha_0 u}
  - \frac12\norm{\alpha_0 u}_x^2
  &=
  \alpha_0 \br{\abr{\nabla f(x), u}
  - \frac{\alpha_0}2}
  \\
  &\geq
  \alpha_0 \br{\abr{\nabla f(x), u}
  - \frac{\abr{\nabla f(x), u}}2}
  \\
  &=
  \frac12 \abr{\nabla f(x), \alpha_0 u}
  \\
  &=
  \frac12 \max_{0\leq \alpha \in \cG_x}\abr{\nabla f(x), \alpha u}
  \\
  &=
  \frac12 \max_{g \in \cG_x(u)}\abr{\nabla f(x), g}
  .
\end{align*}
This completes the proof.
\end{proof}
We are now ready for the proof of our main theorem.

\begin{proof}[Proof of \cref{thm:csd}]
	For ease of presentation, we prove for the case that $\epsvgd=0$; the general case follows immediately by replacing $f^\star$ with the error floor $f^\star + \epsvgd$ everywhere in the proof.
    Throughout the proof we 
    denote $\cG_k \eqq \cG_{x_k}$ and $\Des_k \eqq \Des_{x_k}$. We recall these are the set of potential gradient mappings and steepest descent magnitude \cref{def:csd_gm_set,def:csd_dm}. For convenience:
\begin{align*}
	\cG_t = \cbr{\frac1\eta(x_k - y) \mid y \in \cX};\quad 
    \Des_{k} = \max_{g \in \cG_t} \cbr{ \abr{\nabla f(x_k), g} 
    - \frac{1}{2}\norm{g}_{x_k}^2}.
\end{align*}
    By \cref{lem:csd_gt_step}, we have that for all $k$, $g_k \eqq \frac1\eta(x_k - x_{k+1})$ satisfies
    \begin{align*}
        g_k \in \argmaxeps[(\epserr/\eta)]_{g \in \cG_k} \cbr{ \abr{\nabla f(x_k), g} 
		- \frac{1}{2}\norm{g}_{x_k}^2}.
    \end{align*}
    Hence, by smoothness of $f$:
\begin{align*}
	f(x_{k+1})
	&\leq f(x_k) + \abr{\nabla f(x_k), x_{t=k+1} - x_k} + \frac \beta2\norm{x_{k+1} - x_k}_{x_k}^2
	\\
	&= f(x_k) - \eta \abr{\nabla f(x_k), g_k} + \frac {\eta^2\beta}2\norm{g_k}_{x_k}^2
	\\
	&\leq f(x_k) - \eta \abr{\nabla f(x_k), g_k} + \frac {\eta}2\norm{g_k}_{x_k}^2
	\tag{$\eta\leq1/\beta$}
	\\
	&= f(x_k) - \eta \br{\abr{\nabla f(x_k), g_k} - \frac {1}2\norm{g_k}_{x_k}^2}
	\\
	&\leq f(x_k) - \eta \br{\Des_k - \epserr/\eta},
\end{align*}
which implies
\begin{align}\label{eq:csd_vgd_descent}
    \eta \Des_k \leq f(x_k) - f(x_{k+1}) + \epserr.
\end{align}
Further, by the VGD assumption and \cref{lem:csd_to_stationarity}:
\begin{align*}
    \frac{1}{\cvgd}\br{f(x_k) - f^\star}
    \leq 
    \max_{y\in \cX}\abr{\nabla f(x_k), x_k - y}
    \leq 2 D \max\cbr{\Des_k, \sqrt \Des_k},
\end{align*}
hence, with $E_k \eqq f(x_k) - f^\star$ the above implies
\begin{align}\label{eq:csd_vgd_1}
    \frac{1}{(2 D \cvgd)^2} E_k^2
    \leq\max\cbr{\Des_k^2, \Des_k}
    .
\end{align}
In addition, we have
\begin{align*}
    \Des_k &= \max_{g\in \cG_k} \cbr{
        \abr{\nabla f(x_k), g} - \frac12\norm{g}_{x_k}^2
    }
    \\
    &\leq 
    \max_{g\in \R^d} \cbr{
        \abr{\nabla f(x_k), g} - \frac12\norm{g}_{x_k}^2
    }
    = \max_{\norm{u}_{x_k}=1}\abr{\nabla f(x_k), u}^2
    = \br{\norm{\nabla f(x_k)}_{x_k}^*}^2
    \leq M^2
    ,
    \\
    \implies
    \max\cbr{\Des_k^2, \Des_k}
    &\leq \max\cbr{M^2 \Des_k, \Des_k}
    =
    M^2 \Des_k.
    \tag{$M\geq 1$}
\end{align*}
Combining with \cref{eq:csd_vgd_1} we obtain
\begin{align*}
    \frac{1}{(2 D \cvgd)^2} E_k^2
    \leq M^2 \Des_k
    \implies 
    \br{2 M \cvgd D}^{-2} E_k^2
    \leq \Des_k.
\end{align*}
Further combining the above display with \cref{eq:csd_vgd_descent}, we obtain:
\begin{align*}
    \omega_0 E_k^2
    &\leq \eta \Des_k
    \leq f(x_k) - f(x_{k+1}) + \epserr
    = E_k - E_{k+1} + \epserr
    \\
    \iff
    \omega_0 \br{E_k^2 - \epserr/\omega_0}
    &\leq E_k - E_{k+1}
    .
\end{align*}
for $\omega_0 \eqq \eta \br{2 \cvgd M D}^{-2}$.
We now consider two cases. In the first, the algorithm converges to the error floor determined by $\epserr$, in the second, $E_k^2 \geq 2\epserr/\omega_0$ for all $k$.
\paragraph{Case 1 (Convergence to error floor).} Suppose that
$E_{k_0}^2 \leq 2\epserr/\omega_0$ for some $k_0\in [K]$. By our previous display, it holds that for all $k$,
\begin{align*}
    E_{k+1}
    &\leq E_k - \omega_0 \br{E_k^2 - \epserr/\omega_0} 
    ,
\end{align*}
which implies that whenever $E_k^2 \geq 2\epserr/\omega_0$, $E_{k+1}\leq E_k$.
Further, since $E_k\geq 0$, we also have that in any case, $E_{k+1}\leq E_k + \epserr$.
Now suppose by contradiction that $E_{K+1} > \sqrt{2\epserr/\omega_0} + \epserr$.
This implies that the last iteration was a descent iteration, hence $E_K > \sqrt{2\epserr/\omega_0} + \epserr$. Proceeding with this argument inductively contradicts our assumption that $E_{k_0}^2 \leq 2\epserr/\omega_0$. Thus, we obtain
\begin{align*}
E_{K+1} \leq \sqrt{2\epserr/\omega_0} + \epserr \leq \frac{4  \cvgd M D }{\sqrt \eta}\sqrt{\epsilon}.	
\end{align*}
\paragraph{Case 2 (Descent throughout).} Suppose that ``Case 1'' does not occur, then
$E_k^2 \geq 2\epserr/\omega_0$ for all $k$, which implies $\eta\Des_k - \epsilon \geq \eta\Des_k/2$, and for $\omega\eqq\omega_0/2$:
\begin{align*}
    \omega E_k^2
    \leq E_t - E_{t+1}
    .
\end{align*}
Now, divide both sides of the previous display by $E_k E_{k+1}$ and use that $E_k\geq E_{k+1}$,
\begin{align*}
    \omega \leq \frac{\omega E_k}{E_{k+1}}
    \leq \frac1{E_{k+1}} - \frac1{E_{k}} 
    ,
\end{align*}
and sum over $k$, telescoping the RHS to obtain
\begin{align*}
    \omega K 
    \leq \frac1{E_{K +1}} - \frac1{E_{1}}
    &\implies
    \omega K (E_{K+1} E_1) 
    \leq E_1 - E_{K+1}
    \\
    &\implies
    E_{K+1} 
    \leq E_1 - \omega (E_{K+1} E_1)K.
\end{align*}
Finally,
\begin{align*}
    0\leq E_{K+1} 
    \leq E_1\br{1 - \omega E_{K+1}K},
\end{align*}
and dividing by $E_1$ (if $E_1=0$, there is nothing to prove), we obtain
\begin{align*}
    0\leq  1 - \omega E_{K+1}K
    \implies
    E_{K+1} \leq \frac{1}{\omega K}
    =
    \frac{2\br{2 \cvgd M D}^{2}}{\eta K},
\end{align*}
and which completes the proof.
\end{proof}

Next, we additionally provide a proof for convergence to an approximate stationary point without the VGD assumption. Here we prove for the error free case
\begin{theorem}
    Assume that $f\colon \cX \to \R$ is $\beta$-smooth w.r.t.~a norm $\norm{\cdot}$ over $\cX$ and attains a minimum $f^\star=\min_{x\in \cX} f(x)$.
    Then the constrained steepest descent method \cref{def:csd} with step size $\eta\leq 1/\beta$ guarantees that after $K \geq 1$ iterations, we have that for some $k\in [K]$, $x_k$ is an approximate stationary point:
    \begin{align*}
        \min_{y \in \cX} \abr{\nabla f(x_k), y - x_k}
        \geq - 2D
        \max\cbr{\frac{E_1}{\eta K} + \epserr/\eta, \sqrt {\frac{E_1}{\eta  K}+\epserr/\eta}},
    \end{align*}
    where $E_1\eqq f(x_1) - f^\star$.
\end{theorem}
\begin{proof}
	Similar to the proof of \cref{thm:csd}, we obtain for all $k$:
	\begin{align}
	    \eta \Des_k \leq f(x_k) - f(x_{k+1}) + \epserr.
	\end{align}
Summing over $k$ and rearranging,
\begin{align*}
	\sum_{t=1}^T \Des_k
	\leq \frac{f(x_1) - f(x_{T+1})}{\eta}
	+ \frac{K \epserr}{\eta},
\end{align*}
which implies that for some $k$,
$$
	\Des_k \leq \frac{f(x_1) - f(x_\star)}{\eta K} + \frac{\epserr}{\eta}.
$$
By \cref{lem:csd_to_stationarity}, we have
$$
	\max_{y\in \cX}\abr{\nabla f(x_k), x_k - y}
	\leq 2 D \max\cbr{\Des_k, \sqrt \Des_k}
	\leq 2 D \max\cbr{\frac{E_1}{\eta K} + \epserr/\eta, \sqrt {\frac{E_1}{\eta  K}+\epserr/\eta}},
$$
which proves our claim.
\end{proof}

\subsection{Frank-Wolfe method}
We first present the guarantee of the standard FW \citep{frank1956algorithm} method subject to the VGD condition, without local norms and without a second approximation step. The analysis is fairly standard, but uses the VGD condition where convexity is normally used.
\begin{algorithm}[ht!]
    \caption{Approximate Frank-Wolfe} 
    \label{alg:fw}
	\begin{algorithmic}
	    \STATE \textbf{input:} $
	    	\eta_1, \ldots, \eta_K
	    $; error tolerance $\epsilon>0$.
            \FOR{$k=1, \ldots, K$}
                \STATE Compute
                $\tilde x_{k+1} \in \argmineps[\epsilon]_{x \in \cX} \abr{x, \nabla f(x_k)}$
                \STATE Set
                $x_{k+1} = (1-\eta_k) x_k + \eta_k \tilde x_{k+1}$.
            \ENDFOR
	\end{algorithmic}
\end{algorithm}

\begin{theorem}\label{thm:fw}
Let $\norm{\cdot}$ be a norm, and 
$\cX\subset \R^d$ be a set of bounded diameter $D\geq \max_{x, y\in \cX}\norm{x - y}$.
Assume that $f\colon \cX \to \R$ is $\beta$-smooth w.r.t.~$\norm{\cdot}$ over ${\rm conv} (\cX)$ and attains a minimum $f^\star=\min_{x\in \cX} f(x)$.
Assume further that $\cX, f$ satisfy a $(\cvgd, \epsvgd)$-VGD condition.
    Then the FW method (\cref{alg:fw}) guarantees for all $k$:
    \begin{align*}
	f(x_{k+1}) - f^\star
	&\leq \prod_{s=1}^k(1 - \eta_s/\cvgd)\br{f(x_1) - f^\star} 
+ \frac12\sum_{s=1}^k \prod_{s'=s+1}^k(1 - \eta_{s'}/\cvgd) \br{\eta_s^2 \beta D^2 
     + 2\eta_s \epsilon  
     }     +\epsvgd.
    \end{align*}
    Furthermore, with $\eta_k=\frac{2\cvgd}{k+2}$, we obtain after $K$ iterations:
    \begin{align*}
	f(x_{{K+1}}) - f^\star
	&\leq \frac{ f(x_1) - f^\star + 2 \cvgd^2\beta D^2}{K+2} 
      + 2\cvgd \epsilon 
     +\epsvgd   .
    \end{align*}
\end{theorem}
\begin{proof}
For ease of presentation, we prove for the case that $\epsvgd=0$; the general case follows immediately by replacing $f^\star$ with the error floor $f^\star + \epsvgd$ everywhere in the proof.
    Observe,
	\begin{align*}
		f(x_{k+1}) - f(x_k) 
		&\leq \nabla f(x_k)\T (x_{k+1} - x_k) 
		+ \frac{\beta}{2}\norm{x_{k+1} - x_k}^2
		\tag{$\beta$-smoothness}
		\\
		&= \eta_k \nabla f(x_k)\T (\tilde x_{k+1} - x_k) 
		+ \frac{\beta\eta_k^2}{2}\norm{\tilde x_{k+1} - x_k}^2
		\tag{$x_{k+1} - x_k = \eta_k(\tilde x_{k+1} - x_k)$}
		\\
		&\leq \eta_k \nabla f(x_k)\T (\tilde x_{k+1} - x_k) 
		+ \frac{\eta_k^2\beta D^2}{2}
		.
	\end{align*}
	Further by definition of the algorithm,
	\begin{align*}
		\eta_k \nabla f(x_k)\T (\tilde x_{k+1} - x_k)
		&\leq 
		\eta_k \min_{\tilde x\in \cX}
		\nabla f(x_k)\T (\tilde x - x_k)
		+ \epsilon \eta_k
		\\
		&=
		\frac{\eta_k}{\cvgd} \cvgd\min_{\tilde x\in \cX}
		\nabla f(x_k)\T (\tilde x - x_k)
		+ \epsilon \eta_k
		\\
		&\leq 
		\frac{\eta_k}{\cvgd} \br{f^\star - f(x_k)}
		+ \epsilon \eta_k,
	\end{align*}
	where the last inequality follows by the VGD assumption. Combining this with our previous display now yields,
	\begin{align*}
		f(x_{k+1}) - f(x_k) 
		&\leq \frac{\eta_k}{\cvgd} 
		\br{f^\star - f(x_k)}
		+ \eta_k \epsilon
		+ \beta \eta_k^2 D^2/2,
	\end{align*}
	hence, letting $E_k\eqq f(x_k) - f^\star$ we have,
	\begin{align*}
		E_{k+1}
		&\leq \br{1-\frac{\eta_k}{\cvgd}}E_k
		+ \eta_k^2 \beta D^2/2 
		+ \eta_k \epsilon.
	\end{align*}
    Now, apply the above inequality recursively to obtain
    \begin{align*}
    	E_{k+1} 
    	\leq \prod_{s=1}^k(1 - \eta_s/\cvgd)E_1 
        &+ \frac12\sum_{s=1}^k 
            \prod_{s'=s+1}^k(1 - \eta_{s'}/\cvgd)\br{\eta_s^2 \beta D^2}
        \\
        &\quad   + \sum_{s=1}^k 
            \prod_{s'=s+1}^k(1 - \eta_{s'}/\cvgd)\eta_s
            \epsilon
        ,
    \end{align*}
    which proves the first part.
    For the second part, note that
    choosing $\eta_t = \frac{2\cvgd}{t+2}$ gives
\begin{align*}
	\prod_{s=1}^k (1-\eta_s/\cvgd)=\prod_{s=1}^k \frac{s}{s+2}=\frac{1}{(k+1)(k+2)},    
\end{align*}
and,
\begin{align*}
    \prod_{s'=s+1}^k(1 - \eta_{s'}/\cvgd)\eta_s 
	&= \frac{(s+1)(s+2)}{(k+1)(k+2)}\frac{2\cvgd}{(s+2)}
	\leq \frac{2\cvgd}{k+1}
    \\
    \prod_{s'=s+1}^k(1 - \eta_{s'}/\cvgd)\eta_s^2 
	&= \frac{(s+1)(s+2)}{(k+1)(k+2)}\frac{4\cvgd^2}{(s+2)^2}
	\leq \frac{4\cvgd^2}{(k+1)(k+2)}
    .
\end{align*}
Plugging this back into our bound on $E_{k+1} = f(x_{k+1}) - f^\star$, we obtain:
\begin{align*}
    f(x_{k+1}) - f(x^\star) 
    &\leq
    \frac{f(x_1) - f^\star}{(k+1)(k+2)} 
    + \frac12\sum_{s=1}^k \frac{4\cvgd^2 \beta D^2}{(k+1)(k+2)}
    + \sum_{s=1}^k \frac{2\cvgd \epsilon}{k+1}
    \\
    &\leq
    \frac{f(x_1) - f^\star}{(k+1)(k+2)} 
      + \frac{2 \cvgd^2 \beta D^2}{k+2}
      + 2\cvgd \epsilon
    \\ 
    &\leq 
      \frac{f(x_1) - f^\star + 2 \cvgd^2 \beta D^2}{k+2}
      + 2\cvgd \epsilon,
\end{align*}
as claimed.
\end{proof}

Next, we present the guarantee for the doubly approximate version of FW with local norms.
\begin{algorithm}[ht!]
    \caption{Doubly Approximate Frank-Wolfe} 
    \label{alg:da_fw}
	\begin{algorithmic}
	    \STATE \textbf{input:} $
	    	\eta_1, \ldots, \eta_K
	    $; error tolerances $\epsilon, \tilde \epsilon>0$.
            \FOR{$k=1, \ldots, K$}
                \STATE Compute
                $\tilde x_{k+1} \in \argmineps[\epsilon]_{x \in \cX} \abr{x, \nabla f(x_k)}$
                \STATE Compute
                $x_{k+1} \in 
                \argmineps[\tilde \epsilon]_{x \in \cX} \cbr{\norm{x - ((1-\eta_k) x_k + \eta_k \tilde x_{k+1}))}_{x_k}^2}$.
            \ENDFOR
	\end{algorithmic}
\end{algorithm}

\begin{theorem}\label{thm:da_fw}
    Under \cref{assm:opt_setup}, the Doubly Approximate FW
    \cref{alg:da_fw} guarantees, for all $k$:
    \begin{align*}
	f(x_{k+1}) - f^\star
	&\leq \prod_{s=1}^k(1 - \eta_s/\cvgd)\br{f(x_1) - f^\star} 
	\\ 
     &\quad+ \frac12\sum_{s=1}^k \prod_{s'=s+1}^k(1 - \eta_{s'}/\cvgd) \br{\eta_s^2 \beta D^2 
     + 2\eta_s\br{\epsilon + \beta \sqrt {\tilde \epsilon}} 
     + \tilde \epsilon (2M + \beta)}
     + \epsvgd.
    \end{align*}
    Furthermore, with $\eta_k=\frac{2\cvgd}{k+2}$, we obtain after $K$ iterations:
    \begin{align*}
	f(x_{{K+1}}) - f^\star
	&\leq \frac{(2 \cvgd^2 +1)\beta D^2}{K+2} 
      + 2\cvgd \br{\epsilon + \beta \sqrt{\tilde \epsilon}}
      + \tilde \epsilon \br{M +\beta}K
      + \epsvgd .
    \end{align*}
\end{theorem}
\begin{proof}
For ease of presentation, we prove for the case that $\epsvgd=0$; the general case follows immediately by replacing $f^\star$ with the error floor $f^\star + \epsvgd$ everywhere in the proof.
    For all $k$, let
    \begin{align*}
        x_{k+1}^\star = (1-\eta_k) x_k + \eta_t \tilde x_{k+1}.
    \end{align*}
    We have,
    \begin{align*}
		\frac{\beta}{2}\norm{x_{k+1} - x_k}^2_{x_k}
		&= \frac{\beta}{2}\norm{x_{k+1}^\star - x_k + \br{x_{k+1} - x_{k+1}^\star}}_{x_k}^2
		\\
		&\leq \frac{\beta}{2}\norm{x_{k+1}^\star - x_k}_{x_k}^2 
		+ \beta\norm{x_{k+1}^\star - x_k}\norm{x_{k+1} - x_{k+1}^\star}_{x_k}
		+ \frac\beta2\norm{x_{k+1} - x_{k+1}^\star}^2_{x_k}
		\\
		&= \frac{\beta \eta_k^2}{2}
		\norm{\tilde x_{k+1} - x_k}^2_{x_k} 
		+ \beta\eta_k\norm{\tilde x_{k+1} - x_k}_{x_k}
		\norm{x_{k+1} - x_{k+1}^\star}_{x_k}
		+ \frac\beta2\norm{x_{k+1} - x_{k+1}^\star}^2_{x_k}
		\\
		&\leq \frac{\beta \eta_k^2}{2}D^2 
		+ \beta\eta_k\sqrt {\tilde\epsilon}
		+ \frac\beta2\tilde \epsilon.
	\end{align*}
	Hence,
	\begin{align*}
		f(x_{k+1}) - f(x_k) 
		&\leq \nabla f(x_k)\T (x_{k+1} - x_k) 
		+ \frac{\beta}{2}\norm{x_{k+1} - x_k}^2_{x_k}
		\\
		&\leq \nabla f(x_k)\T (x_{k+1}^\star - x_k) 
		+ \tilde \epsilon M 
		+ \frac{\beta}{2}\norm{x_{k+1} - x_k}^2_{x_k}
		\\
		&= \eta_k \nabla f(x_k)\T (\tilde x_{k+1} - x_k) 
		+ \tilde \epsilon M 
		+ \frac{\beta}{2}\norm{x_{k+1} - x_k}^2_{x_k}
		\\
		&\leq \eta_k \nabla f(x_k)\T (\tilde x_{k+1} - x_k) 
		+ \tilde \epsilon M 
		+ \beta \eta_k^2 D^2/2 
		+ \beta\eta_k\sqrt {\tilde \epsilon}
		+ \beta\tilde \epsilon/2
	\end{align*}
	Now, note that by definition of the algorithm,
	\begin{align*}
		\eta_k \nabla f(x_k)\T (\tilde x_{k+1} - x_k)
		&\leq 
		\eta_k \min_{\tilde x\in \cX}
		\nabla f(x_k)\T (\tilde x - x_k)
		+ \epsilon \eta_k
		\\
		&=
		\frac{\eta_k}{\cvgd} \cvgd\min_{\tilde x\in \cX}
		\nabla f(x_k)\T (\tilde x - x_k)
		+ \epsilon \eta_k
		\\
		&\leq 
		\frac{\eta_k}{\cvgd} \br{f^\star - f(x_k)}
		+ \epsilon \eta_k,
	\end{align*}
	where the last inequality follows by the VGD assumption. Combining this with our previous display now yields,
	\begin{align*}
		f(x_{k+1}) - f(x_k) 
		&\leq \frac{\eta_k}{\cvgd} 
		\br{f^\star - f(x_k)}
		+ \eta_k \epsilon
		+ \tilde \epsilon M 
		+ \beta \eta_k^2 D^2/2 
		+ \beta\eta_k\sqrt {\tilde \epsilon}
		+ \beta\tilde \epsilon/2
		\\
		&= \frac{\eta_k}{\cvgd} 
		\br{f^\star - f(x_k)}
		+ \eta_k^2 \beta D^2/2 
		+ \eta_k \br{\epsilon+ \beta\sqrt {\tilde \epsilon}}
		+ \tilde \epsilon \br{M 
		+ \beta/2},
	\end{align*}
	hence, letting $E_k\eqq f(x_k) - f^\star$ we have,
	\begin{align*}
		E_{k+1}
		&\leq \br{1-\frac{\eta_k}{\cvgd}}E_k
		+ \eta_k^2 \beta D^2/2 
		+ \eta_k \br{\epsilon+ \beta\sqrt {\tilde \epsilon}}
		+ \tilde \epsilon \br{M 
		+ \beta/2}.
	\end{align*}
    Now, apply the above inequality recursively to obtain
    \begin{align*}
    	E_{k+1} 
    	\leq \prod_{s=1}^k(1 - \eta_s/\cvgd)E_1 
        &+ \frac12\sum_{s=1}^k 
            \prod_{s'=s+1}^k(1 - \eta_{s'}/\cvgd)\br{\eta_s^2 \beta D^2}
        \\
        &\quad   + \sum_{s=1}^k 
            \prod_{s'=s+1}^k(1 - \eta_{s'}/\cvgd)\eta_s
            \br{\epsilon  + \beta \sqrt{\tilde\epsilon}}
        \\
        &\quad   + \sum_{s=1}^k 
            \prod_{s'=s+1}^k(1 - \eta_{s'}/\cvgd)\tilde \epsilon
            \br{M + \beta/2}
        ,
    \end{align*}
    which proves the first part.
    For the second part, note that
    choosing $\eta_t = \frac{2\cvgd}{t+2}$ gives
\begin{align*}
	\prod_{s=1}^k (1-\eta_s/\cvgd)=\prod_{s=1}^k \frac{s}{s+2}=\frac{1}{(k+1)(k+2)},    
\end{align*}
and,
\begin{align*}
    \prod_{s'=s+1}^k(1 - \eta_{s'}/\cvgd)
	&= \frac{(s+1)(s+2)}{(k+1)(k+2)},
    \\
    \prod_{s'=s+1}^k(1 - \eta_{s'}/\cvgd)\eta_s 
	&= \frac{(s+1)(s+2)}{(k+1)(k+2)}\frac{2\cvgd}{(s+2)}
	\leq \frac{2\cvgd}{k+1}
    \\
    \prod_{s'=s+1}^k(1 - \eta_{s'}/\cvgd)\eta_s^2 
	&= \frac{(s+1)(s+2)}{(k+1)(k+2)}\frac{4\cvgd^2}{(s+2)^2}
	\leq \frac{4\cvgd^2}{(k+1)(k+2)}
    .
\end{align*}
Plugging this back into our bound on $E_{k+1} = f(x_{k+1}) - f^\star$, we obtain, for any $x^\star\in \argmin_{x\in \cX} f(x)$,
\begin{align*}
    f(x_{k+1}) - f(x^\star) 
    &\leq
    \frac{f(x_1) - f(x^\star)}{(k+1)(k+2)} 
    + \frac12\sum_{s=1}^k \frac{4\cvgd^2 \beta D^2}{(k+1)(k+2)}
    + \sum_{s=1}^k \frac{2\cvgd\br{\epsilon + \beta\sqrt{\tilde \epsilon}}}{k+1}
    + \tilde \epsilon \br{M +\beta}k
    \\
    &\leq
    \frac{f(x_1) - f(x^\star)}{(k+1)(k+2)} 
      + \frac{2 \cvgd^2 \beta D^2}{k+2}
      + 2\cvgd \br{\epsilon + \beta \sqrt{\tilde \epsilon}}
    + \tilde \epsilon \br{M +\beta}k
    \\ &\leq 
    \frac{\beta\norm{x_1 - x^\star}_{x^\star}^2}{2(k+1)(k+2)} 
      + \frac{2 \cvgd^2 \beta D^2}{k+2}
      + 2\cvgd \br{\epsilon + \beta \sqrt{\tilde \epsilon}}
      + \tilde \epsilon \br{M +\beta}k
     \tag{$\abr{\nabla f(x^\star), x_1 - x_\star} \geq 0$}
     \\
      &\leq \frac{(2 \cvgd^2 +1)\beta D^2}{k+2} 
      + 2\cvgd \br{\epsilon + \beta \sqrt{\tilde \epsilon}}
      + \tilde \epsilon \br{M +\beta}k,
\end{align*}
as claimed.
\end{proof}

\subsection{Bregman proximal point method}
We first recall the algorithm as presented in \citet{sherman2025convergence}.
Given any convex regularizer $h\colon \R^d \to \R$, we define the set of $\epsilon$-approximate Bregman proximal point update solutions with step-size $\eta > 0$ by:
\begin{align}\label{def:bregman_prox_approx}
	\cT_{\eta}^{\epsilon}(x; h) 
        \eqq 
	\cbr{x^+ \in \cX \mid \forall z\in \cX: 
		\abr{\grad{f}{x} + \frac1\eta\nabla B_h(x^+, x), z - x^+} \geq -\epsilon}.
\end{align}
The approximate Bregman proximal point update of \citet{sherman2025convergence} is defined by:
\begin{align}\label{alg:opt_omd_0}
    k=1, \ldots, K: \quad 
    		x_{k+1} &\in \cT_{\eta}^{\epsilon}(x_k; R_{x_k}).
\end{align}
\begin{theorem}[\cite{sherman2025convergence}]\label{thm:opt_omd_0}
	Consider \cref{assm:opt_setup}, and suppose further that
	the local regularizer $R_x$ is $1$-strongly convex and has an $L$-Lipschitz gradient w.r.t.~$\norm{\cdot}_x$ for all $x\in \cX$.
    Then, for the Bregman proximal point update \cref{alg:opt_omd_0} we have following guarantee, for $\eta\leq 1/(2\beta)$:
	\begin{align*}
		f(x_{K+1}) - f^\star 
		&= O\br{\frac{\cvgd^2 L^2 c_1^2}{\eta K}
        + \cvgd \epsilon 
        + c_1 L \eta^{-\frac12}\sqrt{\epsilon}
            + \epsvgd
            }
        \end{align*}
	where $c_1 \eqq D + \eta M$.
\end{theorem}
We now consider the Bregman proximal point algorithm that operates with iterates that are approximate minimizers in terms of \emph{function values}:
\begin{align}\label{alg:opt_omd}
    k=1, \ldots, K: \quad 
    		x_{k+1} &\gets \argmineps_{x\in \cX}\cbr{
    		\abr{\nabla f(x_k), x} + \frac1\eta B_{R_{x_k}}(x, x_k)
    		}
\end{align}
We will use the following lemma to translate objective sub-optimality to approximate optimality conditions.
\begin{lemma}\label{lem:subopt_to_optcond}
    Let $\norm{\cdot}$ be a norm, and suppose $D\geq \max\cbr{1, \max_{x,y\in \cX}\norm{x - y}}$.
    Let $\phi\colon \cX \to \R$ be $1$-strongly convex and $L$-smooth w.r.t.$\norm{\cdot}$. Then, for any $\epsilon \leq 1$ if $\hat x\in \argmineps[\epsilon]_{x\in \cX}\phi(x)$, we have:
    \begin{align*}
        \abr{\phi(\hat x), y - \hat x}
        \geq - 2L D\sqrt{2\epsilon}.
    \end{align*}
\end{lemma}
\begin{proof}
    Let $x^\star=\argmin_{x\in \cX}\phi(x)$.
    By $1$-strong convexity and our assumption,
    \begin{align*}
        \frac12\norm{\hat x - x^\star}^2
        \leq \phi(\hat x) - \phi(x^\star)
        \leq \epsilon
        \implies
        \norm{\hat x - x^\star}
        \leq \sqrt{2\epsilon}.
    \end{align*}
    Hence, for any $y\in \cX$, by $L$-smoothness:
    \begin{align*}
        \abr{\phi(\hat x), \hat x - y}
        &=
        \abr{\phi(x^\star), \hat x - y}
        + \abr{\phi(\hat x) - \phi(x^\star), \hat x - y}
        \\
        &\leq
        \abr{\phi(x^\star), \hat x - y}
        + L D \sqrt{2\epsilon}
    \end{align*}
    Further by our assumption and optimality conditions at $x^\star$, 
    \begin{align*}
        \abr{\phi(x^\star), \hat x - y}
        &=
        \abr{\phi(x^\star), x^\star - y}
        + \abr{\phi(x^\star), \hat x - x^\star}
        \\
        &\leq \abr{\phi(x^\star), x^\star - y}
        + \phi(\hat x) - \phi(x^\star)
        \\
        &\leq \abr{\phi(x^\star), x^\star - y}
        + \epsilon
        \\
        &\leq
        + \epsilon
        .
    \end{align*}
    Hence, we have for all $y\in \cX$:
    \begin{align*}
        \abr{\phi(\hat x), \hat x - y}
        \leq L D\sqrt{2\epsilon}
        + \epsilon
        \leq 2 L D\sqrt{2\epsilon},
    \end{align*}
    which completes the proof.
\end{proof}
We are now in position to tprove the following.
\begin{theorem}\label{thm:breg_prox}
	In the same setting of \cref{thm:opt_omd_0}, we have that the Bregman proximal method \cref{alg:opt_omd} with $\epsilon\leq 1$ guarantees, for $\eta\leq 1/(2\beta)$:
	\begin{align*}
		f(x_{K+1}) - f^\star 
		&= O\br{\frac{\cvgd^2 L^2 c_1^2}{\eta K}
        + \cvgd L D \sqrt{ \epsilon }
        + \frac{c_1 \sqrt{L^{3}D}}{\sqrt \eta } \epsilon^{1/4}
            + \epsvgd
            }
        \end{align*}
	where $c_1 \eqq D + \eta M$.
\end{theorem}
\begin{proof}
    By \cref{lem:subopt_to_optcond} applied with the norm $\norm{\cdot}_{x_k}$ and $\phi(x)=\abr{\nabla f(x_k), x} + \frac1\eta B_{R_{x_k}}(x, x_k)$, we have that
    $x_{k+1}$ from \cref{alg:opt_omd} satisfies:
\begin{align*}
	\forall z\in \cX: 
		\abr{\grad{f}{x_k} + \frac1\eta\nabla B_{R_{x_k}}(x_{k+1}, x_k), z - x_{k+1}} \geq -2 L D\sqrt{2\epsilon}.
\end{align*}
This implies that $x_{k+1} \in \cT_{\eta}^{\tilde \epsilon}(x_k; R_{x_k})$ for $\tilde \epsilon = 2 L D\sqrt{2\epsilon}$.
This proves the claimed result be substituting $\epsilon \to 2 L D\sqrt{2\epsilon}$ in the bound of \cref{thm:opt_omd_0}.
\end{proof}

\section{Proofs for Section 4}
In this section, we provide the proofs of convergence for our algorithms in their idealized versions. 
All proofs build on casting our algorithms as instances of those presented in the pure optimization setup, in \cref{sec:opt}.
We note that we did not make particular effort in optimizing dependence on problem parameters other than $K$, and in some cases intentionally opted for slightly worse dependence in favor of cleaner bounds.
\subsection{Analysis preliminaries}
Given a state probability measure $\mu\in \Delta(\cS)$, and an action space norm $\norm{\cdot}_\circ\colon \R^A \to \R$, we define the induced state-action weighted $L^p$ norm $\norm{\cdot}_{L^p(\mu), \circ}\colon \R^{SA}\to \R$:
\begin{align}
	\norm{u}_{L^p(\mu), \circ}
	&\eqq \br{\E_{s \sim \mu}\norm{u_s}_\circ^p}^{1/p}.
\end{align}
For any norm $\norm{\cdot}$, we let $\norm{\cdot}^*$ denote its dual. When discussing a generic norm and there is no risk of confusion, we may use $\norm{\cdot}_*$ to refer to its dual.
In addition, for $\mu\in \R^S, Q\in \R^{SA}$, we define the state to state-action element-wise product $\mu \circ Q \in \R^{SA}$:
\begin{align}
	\br{\mu \circ Q}_{s, a} \eqq \mu(s)Q_{s, a}.
\end{align}
Below, we collect a number of results that will be used repeatedly in the analyses.
\begin{lemma}[Value difference; \citealp{kakade2002approximately}]\label{lem:value_diff}
    For any $\rho\in \Delta(\cS)$,
    \begin{align*}
        V_\rho\br{\tilde \pi} -  V_\rho\br{ \pi}
        = \frac{1}{1-\gamma}
        \E_{s\sim \mu_\rho^\pi}\abr{Q^{\tilde \pi}_s, \tilde \pi_s - \pi_s}.
    \end{align*}
\end{lemma}

\begin{lemma}[Policy gradient theorem; \citealp{sutton1999policy}]\label{lem:value_pg}
    For any $\rho\in \Delta(\cS)$,
    \begin{align*}
        \br{\nabla V_\rho(\pi)}_{s, a}
        &= \frac{1}{1-\gamma} \mu^\pi_\rho(s) Q^\pi_{s, a},
        \\
        \abr{\nabla V_\rho(\pi), \tilde \pi - \pi}
        &= \frac{1}{1-\gamma} \E_{s\sim \mu^\pi_\rho} 
        \abr{Q^\pi_s, \tilde \pi_s - \pi_s}.
    \end{align*}
\end{lemma}

\begin{lemma}[\citealp{sherman2025convergence}]\label{lem:value_local_smoothness}
    Let $\pi\colon \cS \to \Delta(\cA)$ be any policy such that $\epsexpl\eqq \min_{s, a}\cbr{\pi_{sa}} > 0$.
    Then, for any $\tilde \pi \in \cS \to \Delta(\cA)$, we have:
    \begin{align*}
        &\av{V(\tilde \pi) - V(\pi) - \abr{\nabla V(\pi), \tilde \pi - \pi}}
        % \\ &\leq %icml_edit
        \leq 
        \min\cbr{
        \frac{  H^3}{\sqrt \epsilon}
            \norm{\tilde \pi - \pi}_{L^2(\mu^\pi), 1}^2,
        \frac{ A H^3}{\sqrt \epsilon}
            \norm{\tilde \pi - \pi}_{L^2(\mu^\pi), 2}^2
        }.
    \end{align*}
\end{lemma}

\begin{lemma}[\citealp{sherman2025convergence}]\label{lem:epsgreedy_vgd}
	Assume $\Pi$ is $(\cvgd, \epsvgd)$-VGD w.r.t.~$\cM$, and consider the $\epsexpl$-greedy exploratory version of $\Pi$, $\Pi^\epsexpl \eqq \cbr{(1-\epsexpl)\pi + \epsexpl u \mid \pi \in \Pi}$, where $u_{s,a} \equiv 1/A$. Then $\Pi^\epsexpl$ is $(\cvgd, \tildepsvgd)$-VGD with $\tildepsvgd \eqq \epsvgd + 12 \cvgd H^2 A \epsexpl $. 
\end{lemma}
The next lemmas follow from standard arguments, for proofs refer to \cite{sherman2025convergence}.
\begin{lemma}\label{lem:weighted_norm_dual}
	For any strictly positive measure $\mu\in \R_{++}^S$,
	the dual norm of $\norm{\cdot}_{L^2(\mu),\circ}$ is given by
	\begin{align}
		\norm{z}_{L^2(\mu), \circ}^*
		&= \sqrt{\int \mu(s)^{-1}\br{\norm{z_s}_\circ^*}^2 {\rm d}s}
	\end{align}
\end{lemma}

\begin{lemma}\label{lem:wnorm_dual_mu}
	Let $\mu\in \Delta(\cS)$, and consider the state-action norm $\norm{\cdot}_{L^2(\mu), \circ}$. For any $W\in \R^{SA}$, we have
	\begin{align*}
		\norm{\mu \circ W}_{L^2(\mu), \circ}^*
		= \sqrt{\E_{s\sim \mu} \br{\norm{W_s}_\circ^*}^2}
	\end{align*}
\end{lemma}
\begin{lemma}\label{lem:value_local_lip}
	For any policy $\pi\in \Delta(\cA)^\cS$ and action norm $\norm{\cdot}_\circ\colon \R^\cA \to \R$, if $\max_{s\in \cS}\norm{Q^\pi_s}_\circ^*\leq G$, then it holds that:
	\begin{aligni*}
            \norm{\nabla V(\pi)}_{L^2(\mu^\pi), \circ}^*
            \leq H G.
        \end{aligni*}
\end{lemma}
\begin{proof}
Observe, by \cref{lem:wnorm_dual_mu}:
	\begin{align*}
            \norm{\nabla V(\pi)}_{L^2(\mu^\pi), \circ}^*
            =
            H\norm{\mu^\pi \circ Q^\pi}_{L^2(\mu^\pi), \circ}^*
            =
            H\sqrt{\E_{s \sim \mu^\pi} \br{\norm{Q^\pi_s}_\circ^*}^2}
            \leq H G.
        \end{align*}
\end{proof}

\subsection{SDPO}
\begin{theorem*}[Restatement of \cref{thm:sdpo}]
    Let $\Pi$ be a convex policy class that satisfies $(\cvgd, \epsvgd)$-VGD w.r.t.~$\cM$.
    Suppose that SDPO (\cref{alg:sdpo}) is executed with the $L^1$ action norm $\norm{\cdot}_1$.
    Then, after $K$ iterations, with appropriately tuned $\eta$ and $\epsexpl$, the output of SDPO satisfies:
    \begin{align*}
        V(\hat \pi) - V^\star(\Pi)
        &= O\br{
        \frac{ \cvgd^2 A H^{4}}{ K^{2/3}}
        + \cvgd H^3 \sqrt A K^{1/6}\sqrt{\varepsilon}
        + \epsvgd}.
    \end{align*}
\end{theorem*}
\begin{proof}[Proof of \cref{thm:sdpo}]
	By the policy gradient theorem (\cref{lem:value_pg}), the update step in \cref{alg:sdpo} may be equivalently written as:
	\begin{align}\label{eq:sdpo_ideal}
        \pi^{k+1}
        \in  
        \argmineps_{\pi\in \Pi^\epsexpl}
            \abr{\nabla V(\pi^k), \pi} 
            + \frac1{2\eta} \norm{\pi - \pi^k}_{L^2(\mu^k),1}^2
   \end{align}
   We now verify a number of conditions that place us in the setup of \cref{assm:opt_setup}.
   \begin{itemize}
   	\item 	\textbf{Local smoothness.}
	By \cref{lem:value_local_smoothness} and the definition of $\Pi^\epsexpl$, the value function is $(2\sqrt A H^3/\sqrt{\epsexpl})$ locally smooth w.r.t.~the 
	local norm $\pi \mapsto \norm{\cdot}_{L^2(\mu^\pi),1}$. 
	\item 	\textbf{VGD condition for $\Pi^\epsexpl$.}
	By \cref{lem:epsgreedy_vgd}, we have that $\Pi^\epsexpl$ satisfies $(\cvgd, \tildepsvgd)$ with $\tildepsvgd \eqq \epsvgd + 12 \cvgd H^2 A \epsexpl$.
	\item \textbf{Local Lipschitz property.} By \cref{lem:value_local_lip}, the value function is $H^2$-local Lipschitz w.r.t.~the 
	local norm $\pi \mapsto \norm{\cdot}_{L^2(\mu^\pi),1}$. 
	\item \textbf{Diameter bound.} We have that $\norm{\pi' - \tilde \pi}_{L^2(\mu^\pi),1} \leq \max_{p, q\in \Delta(\cA)}\norm{p-q}_1\leq 2$, for all $\pi, \pi', \tilde\pi\in \Delta(\cA)^\cS$.
   \end{itemize}
   The above imply we are in the setting of \cref{thm:csd} 
   with $\beta=2\sqrt A H^3/\sqrt{\epsexpl}$, $M=H^2, D=2$.
   Thus, setting $\eta=\sqrt{\epsexpl}/(2H^3 \sqrt A)$, 
	ensures that after $K$ iterations of \cref{eq:sdpo_ideal}, it is guaranteed that:
	\begin{align*}
        V(\pi^{K+1}) - V^\star(\Pi)
        &\lesssim
        \frac{\br{ \cvgd H^2 }^{2}}{\eta K}
        + \frac{  \cvgd H^2 }{\sqrt \eta}\sqrt{\epserr}
        + \tildepsvgd
        \\
        &\lesssim
        \frac{ \cvgd^2 \sqrt A H^{7}}{\sqrt{\epsexpl} K}
        + \frac{\cvgd H^5 A^{1/4}}{\epsexpl^{1/4}}\sqrt{\epserr}
        + \nu A H^2 \epsexpl
        +\epsvgd,
    \end{align*}
	where $\lesssim$ hides only universal constant factors.
	Now set $\epsexpl = \frac{ H^{2}}{ K^{2/3}}$, then
	\begin{align*}
	        V(\pi^{K+1}) - V^\star(\Pi)
	        &\lesssim
	        \frac{ \cvgd^2 A H^{6}}{ K^{2/3}}
	        + \frac{\cvgd H^5 A^{1/4}}{\epsexpl^{1/4}}\sqrt{\epserr}
	        +\epsvgd
	        \\
	        &\leq
	        \frac{ \cvgd^2 A H^{6}}{ K^{2/3}}
	        + \cvgd H^5 \sqrt A K^{1/6}\sqrt{\epserr}
	        +\epsvgd,
	\end{align*}
	as required.
\end{proof}

\subsection{CPI and DA-CPI}
\label{sec:proofs_cpi_dacpi}
In this section, we provide the analysis for CPI and DA-CPI. For CPI, we can make an argument using a non-local norm owed to the usage of the exact convex combination policy obtained in the second step if each iteration in \cref{alg:cpi}. Our first lemma below establishes smoothnes of the value function w.r.t.~the global $\norm{\cdot}_{\infty, 1}$ norm.

\begin{lemma}\label{lem:value_smoothness_infty_1}
    The value function is $(2 H^3)$-smooth w.r.t. the $\norm{\cdot}_{\infty, 1}$ norm; for any $\pi, \tilde \pi \in \cS \to \Delta(\cA)$, we have:
    \begin{align*}
	\av{V(\tilde \pi) - V(\pi) - \abr{\nabla V(\pi), \tilde \pi - \pi}}
	&\leq 
        \frac{2H^3}{2} \norm{\tilde \pi - \pi}_{\infty, 1}^2.
    \end{align*}
\end{lemma}

\begin{proof}
    We have by value difference \cref{lem:value_diff} and the policy gradient theorem \cref{lem:value_pg}:
    \begin{align}
	\av{V(\tilde \pi) - V(\pi) - \abr{\nabla V(\pi), \tilde \pi - \pi}}
	&= \av{H\E_{s \sim \mu^\pi}\abr{Q^{\tilde \pi}_s, \tilde \pi_s - \pi_s}
	- H\E_{s \sim \mu^\pi}\abr{Q^{\pi}_s, \tilde \pi_s - \pi_s}}
	\nonumber \\
        &= \av{H\E_{s \sim \mu^\pi}\abr{Q^{\tilde \pi}_s - Q^\pi_s, \tilde \pi_s - \pi_s}}
	\nonumber \\
	&\leq
		H\E_{s \sim \mu^\pi}\sbr{\av{\abr{{Q^{\tilde \pi}_s - Q^\pi_s}, 
		\tilde \pi_s - \pi_s}}}
	\nonumber \\
	&\leq
		H\E_{s \sim \mu^\pi}\sbr{
			\norm{Q^{\tilde \pi}_s - Q^\pi_s}_\infty
		\norm{\tilde \pi_s - \pi_s}_1}
        \label{eq:lem_smooth_1}.
    \end{align}
    Further, again by value difference \cref{lem:value_diff}, for any $s, a \in \cS \times \cA$:
    \begin{align*}
	Q_{s, a}^{\tilde \pi} - Q_{s, a}^\pi
	&= \gamma \E_{s' \sim \P_{s, a}}\sbr{V_{s'}(\tilde \pi) - V_{s'}(\pi)}
	\\
	&= \gamma H\E_{s' \sim \P_{s, a}}
            \sbr{\E_{s''\sim \mu^\pi_{s'}}\abr{Q^{\tilde\pi}_{s''}, \tilde \pi_{s''} - \pi_{s''}}}
	\\
	&= \gamma H\sum_{s'} \P(s' |s, a)
		\sum_{s''} \mu^\pi_{s'}(s'')
	\abr{Q^{\tilde\pi}_{s''}, \tilde \pi_{s''} - \pi_{s''}}
	\\
	&= \gamma H\sum_{s''} \sum_{s'} \P(s' |s, a)
            \mu^\pi_{s'}(s'')
	\abr{Q^{\tilde\pi}_{s''}, \tilde \pi_{s''} - \pi_{s''}}
	\\
	&= \gamma H\sum_{s''} \mu^\pi_{\P_{s, a}}(s'')
	\abr{Q^{\tilde\pi}_{s''}, \tilde \pi_{s''} - \pi_{s''}}
        \\
	&= \gamma H\E_{s'' \sim \mu^\pi_{\P_{s, a}}} 
	\abr{Q^{\tilde\pi}_{s''}, \tilde \pi_{s''} - \pi_{s''}}.
    \end{align*}
    This implies that for any $s$,
    \begin{align*}
        \norm{Q^{\tilde \pi}_s - Q^\pi_s}_\infty
        &=
        \gamma H\max_a\av{
            \E_{s' \sim \mu^\pi_{\P_{s, a}}} 
    	\abr{Q^{\tilde\pi}_{s'}, \tilde \pi_{s'} - \pi_{s'}}
        }
        \\
        &\leq
        \gamma H^2 \max_a
            \E_{s' \sim \mu^\pi_{\P_{s, a}}} 
    	\norm{\tilde \pi_{s'} - \pi_{s'}}_1
        \\
        &\leq
        \gamma H^2 \norm{\tilde \pi - \pi}_{\infty,1}
    \end{align*}
    Plugging the above back into \cref{eq:lem_smooth_1}, we obtain
    \begin{align*}
	\av{V^{\tilde \pi} - V^{\pi} - \abr{\nabla V^\pi, \tilde \pi - \pi}}
	&\leq
		\gamma H^3 \norm{\tilde \pi - \pi}_{\infty,1} \E_{s \sim \mu^\pi}
		\norm{\tilde \pi_s - \pi_s}_1
        \\
        &\leq 
        \gamma H^3 \norm{\tilde \pi - \pi}_{\infty, 1}^2,
    \end{align*}
    which completes the proof up to a trivial computation.
\end{proof}
We are now ready to prove the guarantee for CPI (\cref{alg:cpi}), by means of reducing it to a non-Euclidean instance of FW \cref{alg:fw}.
\begin{theorem*}[Restatement of \cref{thm:cpi}]
	Let $\Pi$ be a policy class that satisfies $(\cvgd, \epsvgd)$-VGD w.r.t.~$\cM$.
    Suppose that CPI (\cref{alg:cpi}) is executed with the step size choices $\eta_k=\frac{2\cvgd}{k+2}$ for $k=1, \ldots, K$. Then, we have the guarantee that:
    \begin{align*}
    	V(\pi^K) - V^\star(\Pi)
    	\leq \frac{8(2\cvgd^2 + 1)H^3}{K} + 2\cvgd\varepsilon
        + \epsvgd
    \end{align*}
\end{theorem*}
\begin{proof}[Proof of \cref{thm:cpi}]
	By the policy gradient theorem (\cref{lem:value_pg}), the update step in CPI (\cref{alg:cpi}) may be equivalently written as:
	\begin{align*}
        \pi^{k+1}
        \in  
        \argmineps_{\pi\in \Pi}
            \abr{\nabla V(\pi^k), \pi - \pi^k}.
   \end{align*}
   We now verify conditions that place us in the setting of \cref{thm:fw}.
   \begin{itemize}
   	\item 	\textbf{Global smoothness.}
	By \cref{lem:value_smoothness_infty_1}, the value function is $(2H^3)$-smooth w.r.t.~the 
	$\norm{\cdot}_{\infty,1}$ norm. 
	\item 	\textbf{VGD condition.} By assumption, 
		$\Pi$ satisfies $(\cvgd,\epsvgd)$ -VGD.
	\item \textbf{Diameter bound.} We have that $\norm{\pi - \tilde \pi}_{\infty,1} \leq \max_{p, q\in \Delta(\cA)}\norm{p-q}_1\leq 2$ for all $\pi, \tilde \pi\in \Delta(\cA)^\cS$.
   \end{itemize}
   The above imply we are in the setting of \cref{thm:fw} 
   with $\beta=2 H^3$ and $D=2$.
   Thus, our step size choice ensures that after $K$ iterations, it is guaranteed that:
   \begin{align*}
        V(\pi^{K+1}) - V^\star(\Pi)
        &\leq
        \frac{H + 16 \cvgd^2 H^3}{ K}
        + 2\cvgd\epserr
        +\epsvgd
        \leq
        \frac{(16 \cvgd^2 + 1) H^3}{ K}
        + 2\cvgd\epserr
        +\epsvgd,
    \end{align*}
	as required.
\end{proof}

Next, we provide the proof for the guarantees of DA-CPI, which relies on the use of local norms and is actor-oracle efficient.
\begin{theorem*}[Restatement of \cref{thm:da_cpi}]
    Let $\Pi$ be a convex policy class that satisfies $(\cvgd, \epsvgd)$-VGD w.r.t.~$\cM$.
    Suppose that DA-CPI (\cref{alg:da_cpi}) is executed with the $L^1$ action norm $\norm{\cdot}_1$.
    Then, for an appropriate setting of $\eta_1, \ldots, \eta_K$ and $\epsexpl$, we have that the DA-CPI output satisfies:
    \begin{align*}
        V(\hat \pi) - V^\star(\Pi)
        &= O\br{\cvgd^2 AH^3\br{
	   \frac{1}{K^{2/3}} 
            + \varepsilon^{1/3}
    	+ \varepsilon^{2/3}  K^{2/3}}
        }
    .
    \end{align*}
\end{theorem*}
\begin{proof}[Proof of \cref{thm:da_cpi}]
	By the policy gradient theorem (\cref{lem:value_pg}), the first update step in \cref{alg:da_cpi} may be equivalently written as:
	\begin{align*}
        \pi^{k+1}
        \in  
        \argmineps_{\pi\in \Pi^\epsexpl}
            \abr{\nabla V(\pi^k), \pi - \pi^k}.
   \end{align*}
   We now verify a number of conditions that place us in the setup of \cref{assm:opt_setup}.
   \begin{itemize}
   	\item 	\textbf{Local smoothness.}
	By \cref{lem:value_local_smoothness} and the definition of $\Pi^\epsexpl$, the value function is $(2\sqrt A H^3/\sqrt{\epsexpl})$ locally smooth w.r.t.~the 
	local norm $\pi \mapsto \norm{\cdot}_{L^2(\mu^\pi),1}$. 
	\item 	\textbf{VGD condition for $\Pi^\epsexpl$.}
	By \cref{lem:epsgreedy_vgd}, we have that $\Pi^\epsexpl$ satisfies $(\cvgd, \tildepsvgd)$ with $\tildepsvgd \eqq \epsvgd + 12 \cvgd H^2 A \epsexpl$.
	\item \textbf{Local Lipschitz property.} By \cref{lem:value_local_lip}, the value function is $H^2$-local Lipschitz w.r.t.~the 
	local norm $\pi \mapsto \norm{\cdot}_{L^2(\mu^\pi),1}$. 
	\item \textbf{Diameter bound.} We have that $\norm{\pi' - \tilde \pi}_{L^2(\mu^\pi),1} \leq \max_{p, q\in \Delta(\cA)}\norm{p-q}_1\leq 2$, for all $\pi, \pi', \tilde\pi\in \Delta(\cA)^\cS$.
   \end{itemize}
   The above imply we are in the setting of \cref{thm:da_fw} 
   with $\beta=2\sqrt A H^3/\sqrt{\epsexpl}$, $M=H^2, D=2$, and $\epsilon=\tilde \epsilon=\epserr$.
   Thus, with step sizes set according to the statement of \cref{thm:da_fw},
	we have that after $K$ iterations:
	\begin{align*}
        V(\pi^{K+1}) - V^\star(\Pi)
        &\lesssim
        \frac{\cvgd^2 \sqrt A H^3 }{\sqrt{\epsexpl}K}
        + \frac{\cvgd\sqrt A H^3}{\sqrt{\epsexpl}}
        \sqrt{\epserr} 
        + \frac{\sqrt A H^3}{\sqrt{\epsexpl}}\epserr K 
        + \cvgd A H^2 \epsexpl
        + \epsvgd
        \\
        &\leq
        \frac{\cvgd^2 \sqrt A H^3 }{\sqrt{\epsexpl}}
        \br{\frac1K + \sqrt{\epserr} + \epserr K}
        + \cvgd A H^2 \epsexpl
        + \epsvgd
    \end{align*}
	where $\lesssim$ hides only universal constant factors.
	Now, set $\epsexpl=\br{\frac{1}{K} + \sqrt{\epserr}+ \epserr  K}^{2/3}$, and use the fact that $(a + b)^{2/3} \leq a^{2/3} + b^{2/3}$
to immediately obtain the stated bound.
\end{proof}

\subsection{PMD}
\label{sec:proofs_pmd}
The PMD method \citep{tomar2020mirror,xiao2022convergence,lan2023policy} make use of an action regularizer, and more specifically the Bregman divergence w.r.t.~the chosen regularizer, which we define below.
\begin{definition}[Bregman divergence]\label{def:bregman_divergenvce}
    Given a convex differentiable regularizer $R \colon \R^\cA \to \R$, the Bregman divergence w.r.t.~$R$ is:
    \begin{align*}
        B_R(u, v) \eqq R(u) - R(v) - \abr{\nabla R(v), u -v}.
    \end{align*}
\end{definition}
We will make use of the following elementary lemma, which follows from standard arguments; for proof see \citet{sherman2025convergence}.
\begin{lemma}\label{lem:reg_transform}
	Assume $h\colon \R^A \to \R$ is $1$-strongly convex and has $L$-Lipschitz gradient w.r.t.~$\norm{\cdot}$.
	Let  $\mu\in \Delta(\cS)$, and define $R_\mu(\pi) \eqq \E_{s\sim \mu}[h(\pi_s)]$.
	Then
	\begin{enumerate}
		\item $B_{R_\mu}(\pi, \tilde \pi) = \E_{s \sim \mu} B_R(\pi_s, \tilde \pi_s)$.
		\item $R_\mu$ is $1$-strongly convex and has an $L$-Lipschitz gradient w.r.t. $\norm{\cdot}_{L^2(\mu), \circ}$. 
	\end{enumerate}

\end{lemma}
Below we restate and prove the guarantee for the PMD method detailed in \cref{alg:pmd}.

\begin{algorithm}[tb]
   \caption{Policy Mirror Descent (PMD)}
   \label{alg:pmd}
\begin{algorithmic}
   \STATE {\bfseries Input:} 
   $K\geq 1, \eta > 0, \varepsilon > 0, \epsexpl >0, \Pi\in \Delta(\cA)^\cS$, and action regularizer $R\colon\R^\cA \to \R$.
   \STATE Initialize $\pi^1 \in \Pi^\epsexpl$
   \FOR{$k=1$ {\bfseries to} $K$}
   \STATE Update
   \begin{aligni*}
        \pi^{k+1}
        \gets 
        \argmineps_{\pi\in \Pi^\epsexpl}
        \E_{s\sim \mu^k} \sbr{
            H\abr{Q^{k}_{s}, \pi_s} 
            + \frac1{\eta} B_R\br{\pi_s, \pi_s^k}
        }
   \end{aligni*}
   \ENDFOR
   \RETURN $\hat \pi \eqq \pi^{K+1}$
\end{algorithmic}
\end{algorithm}

\begin{theorem*}[Restatement of \cref{thm:pmd}]
    Let $\Pi$ be a convex policy class that satisfies $(\cvgd, \epsvgd)$-VGD w.r.t.~$\cM$.
    Suppose that PMD (see \cref{alg:pmd}) is 
    executed with the $L^2$ action regularizer.
    Then, with an appropriate tuning of  $\eta, \epsexpl$, we have that the output of PMD satisfies:
    \begin{align*}
        V(\hat \pi) - V^\star(\Pi)
        &= O\br{
            \frac{\cvgd^2 A^{3/2} H^3 }{ K^{2/3}}
            + \br{\cvgd + H^2 A K^{1/6}}\varepsilon^{1/4}
            }.
    \end{align*}
\end{theorem*}
\begin{proof}[Proof of \cref{thm:pmd}]
	We now verify a number of conditions that place us in the setup of \cref{assm:opt_setup}.
   \begin{itemize}
   	\item 	\textbf{Local smoothness.}
	By \cref{lem:value_local_smoothness} and the definition of $\Pi^\epsexpl$, the value function is $(2 A^{3/2} H^3/\sqrt{\epsexpl})$ locally smooth w.r.t.~the 
	local norm $\pi \mapsto \norm{\cdot}_{L^2(\mu^\pi),2}$. 
	\item 	\textbf{VGD condition for $\Pi^\epsexpl$.}
	By \cref{lem:epsgreedy_vgd}, we have that $\Pi^\epsexpl$ satisfies $(\cvgd, \tildepsvgd)$ with $\tildepsvgd \eqq \epsvgd + 12 \cvgd H^2 A \epsexpl$.
	\item \textbf{Local Lipschitz property.} By \cref{lem:value_local_lip} and the fact that $\norm{Q^\pi_s}_2\leq \sqrt A H$ for all $\pi, s$, the value function is $(\sqrt A H^2)$-local Lipschitz w.r.t.~the 
	local norm $\pi \mapsto \norm{\cdot}_{L^2(\mu^\pi),2}$. 
	\item \textbf{Diameter bound.} We have that $\norm{\pi' - \tilde \pi}_{L^2(\mu^\pi),2} \leq \max_{p, q\in \Delta(\cA)}\norm{p-q}_2\leq 2$, for all $\pi, \pi', \tilde\pi\in \Delta(\cA)^\cS$.
	\item \textbf{Regularizer smoothness.} The euclidean action norm $R(p)=\frac12\norm{p}_2^2$ is $1$-smooth.
   \end{itemize}
   The above imply we are in the setting of \cref{thm:breg_prox} 
   with $\beta=2A^{3/2} H^3/\sqrt{\epsexpl}$, $M=\sqrt A H^2, D=2, L=1$. 
   Thus, setting $\eta=\sqrt{\epsexpl}/(2H^3 A^{3/2})$, we have $c_1=D + \eta M=O(1)$, and the guarantee that ($\lesssim$ suppresses constant numerical factors):
	\begin{align*}
		V(\pi^{K+1}) - V^\star(\Pi)
		&\lesssim  
		\frac{\cvgd^2 }{\eta K}
        + \cvgd \sqrt{ \epserr }
        + \frac{ \epserr^{1/4}}{\sqrt \eta }
            + \epsvgd
        \\
        &\lesssim 
        \frac{\cvgd^2 A^{3/2} H^3 }{\sqrt\epsexpl K}
            + \cvgd \sqrt{ \epserr }
            + \frac{\sqrt {H^3 A^{3/2}}}{\epsexpl^{1/4}}\epserr^{1/4}
            + \nu A H^2 \epsexpl
            + \epsvgd
        \\
        &\lesssim 
        \frac{\cvgd^2 A^{3/2} H^3 }{\sqrt\epsexpl K}
            + \epserr^{1/4}\br{\cvgd
            + \frac{H^2 A}{\epsexpl^{1/4}}}
            + \nu A H^2 \epsexpl
            + \epsvgd.
	\end{align*}
	Choosing $\epsexpl = K^{-2/3}$, we immediately obtain the stated bound.
\end{proof}

\section{Sample complexity upper bounds}
\label{sec:sc}
In this section, we demonstrate how our iteration complexity upper bounds may be translated to sample complexity upper bounds.
The sampling scheme \cref{alg:sampler} we employ to estimate the full gradient step is based on importance sampling and in itself is fairly standard. A similar algorithm can be found in e.g., \cite{agarwal2021theory}. Throughout this section we adopt the assumption that $\gamma\leq 1/2$ in sake of simplified presentation. In terms of the effective horizon this implies $H\geq 2$, which is the interesting regime.
\begin{algorithm}[tb]
   \caption{Action-value estimation}
   \label{alg:sampler}
\begin{algorithmic}
   \STATE {\bfseries Input: $\pi$} 
   \STATE Begin rollout at $s_0 \sim \rho_0$
   \STATE For each timestep $t=0,\ldots$, act $a_t \sim \pi_{s_t}$, and
   $\begin{cases}
       \text{continue }
       \quad &\text{w.p.~}\gamma
       \\
       \text{accept } s_t
       \quad &\text{w.p.~}1-\gamma
   \end{cases}$
   \STATE After accepting $s_t$, sample $a_t\sim \Unif(\cA)$ and continue the rollout, terminating at each step w.p.~$1-\gamma$.
   \STATE Assume the rollout terminated at iteration $T$. Define $\widehat Q^\pi_{s_t}\in \R^\cA$ by
   \begin{align*}
    \forall a\in \cA: \;
    \widehat Q^\pi_{s_t, a} = 
    \I\cbr{a=a_t}A\sum_{t'=t}^T r(s_{t'}, a_{t'}).
    \end{align*}
   \RETURN $s_t, \widehat Q^\pi_{s_t}$
\end{algorithmic}
\end{algorithm}
The first lemma given below, provides the connection between optimizing the empirical and population objectives. 
\begin{lemma}\label{lem:update_steps_ge}
    Let $\widetilde \Pi$ be a policy class and suppose $\gamma\leq 1/2$.
    Assume $\Div\colon \R^\cA \times \R^\cA \to \R_+$ is a non-negative function that satisfies:
    \begin{itemize}
    	\item Boundedness: $D \geq \Div(p, q)$ for all $p,q\in \Delta(\cA)$.
    	\item Lipschitz continuity w.r.t.~the $1$-norm:
    $\av{\Div(p, p_0) - \Div(q, p_0)}\leq L\norm{p-q}_1$ for all $p, q, p_0\in \Delta(\cA)$.
    \end{itemize}
    Let $\pi^1\in \widetilde \Pi$, suppose $\pi^{k+1}\in \widetilde \Pi$ satisfy for all $k\in [K]$, for a given learning rate $0<\eta\leq 1$,
    \begin{align}\label{eq:concentration_hat_phi_k}
        \pi^{k+1}
        \in 
        \argmineps[\epserm]_{\pi\in \widetilde \Pi}\cbr{
        \widehat \Phi_k(\pi) \eqq
        \frac1N \sum_{s \in \cD_k} {
            \abr{H \widehat Q^{k}_{s}, \pi_s} 
            + \frac1{\eta} \Div\br{\pi_s, \pi_s^k}
        }},
    \end{align}
    where $\cD_k$ are a (state, action-value) datasets of size $N$ obtained by invoking \cref{alg:sampler}.
    Then,
    for any $\delta>0$, 
    w.p.~$\geq 1-\delta$, it holds that for all $k\in [K]$,
    \begin{align}\label{eq:concentration_phi_k}
        \pi^{k+1}
        \in 
        \argmineps_{\pi\in \widetilde \Pi}\cbr{
        \Phi_k(\pi) \eqq
        \E_{s\sim \mu^k}\sbr{
        \abr{H Q^k_s, \pi_s}+ \frac1{\eta} \Div\br{\pi_s, \pi_s^k}
        }
        },
    \end{align}
    where $\epserr = \epserm + \epsgen$,
    \begin{aligni*}
        \epsgen \eqq \frac{\cgen A H^2 D}{\eta}\sqrt{
            \frac{\log\frac{KN\cN(\epsnet, \widetilde \Pi, \norm{\cdot}_{\infty, 1})}{\delta}}{ N}
        },
    \end{aligni*}
    $\cgen>0$ is an absolute numerical constant, 
    and 
    $\epsnet \geq 
    \frac{\cgen AH^2D}{6 \sqrt{ N} (A H^2\log(2 KN/\delta)  + L)}$. Furthermore, the number of time steps of each episode rolled out by \cref{alg:sampler} is $\leq 2 H \log\br{2 K N/\delta} $.
\end{lemma}
We defer the proof of \cref{lem:update_steps_ge} to \cref{sec:update_steps_ge_proof}.
We now turn
to apply the lemma in conjunction with the iteration complexity guarantee of SDPO (\cref{thm:sdpo}) to obtain a sample complexity upper bound for SDPO in the learning setup (\cref{alg:sc_sdpo}).
Afterwards, we present sample complexity upper bounds for the other algorithms in \cref{sec:sc_cpi_da_cpi,sec:sc_pmd}.
We note that there are a number of places where we expect the analysis can be tightened subject to future work; primarily, the greedy exploration required by the current local-smoothness analysis.
Hence, the actual rate obtained is not the primary focus of our work. Furthermore, we did not make a notable effort in obtaining optimal dependence on all problem parameters, and expect these can be easily tightened by a more careful choice analysis and choice of algorithm input-parameters.

\begin{algorithm}[tb]
   \caption{SDPO in the learning letup}
   \label{alg:sc_sdpo}
\begin{algorithmic}
   \STATE {\bfseries Input:} 
   $K\geq 1, N\geq 1, \eta > 0, \epsexpl >0, \epserm>0, \Pi\in \Delta(\cA)^\cS$, and action norm $\norm{\cdot}_\circ\colon\R^\cA \to \R$.
   \STATE Initialize $\pi^1 \in \Pi^\epsexpl$
   \FOR{$k=1$ {\bfseries to} $K$}
   \STATE Rollout $\pi^k$ for $N$ episodes via \cref{alg:sampler}, obtain $\cD_k = \cbr{s_i^k, \widehat Q_{s_i^k}^k}_{i=1}^N$.
   \STATE Update
   \begin{aligni*}
        \pi^{k+1}
        \gets 
        \argmineps[\epserm]_{\pi\in \Pi^\epsexpl}\cbr{
        \widehat \Phi_k(\pi) \eqq
        \frac1N \sum_{s \in \cD_k} {
            \abr{H \widehat Q^{k}_{s}, \pi_s} 
            + \frac1{2\eta} \norm{\pi_s - \pi_s^k}_\circ^2
        }}
   \end{aligni*}
   \ENDFOR
   \RETURN $\hat \pi \eqq \pi^{K+1}$
\end{algorithmic}
\end{algorithm}
\begin{theorem}\label{thm:sc_sdpo}
	Let $\Pi$ be a convex policy class that satisfies $(\cvgd, \epsvgd)$-VGD w.r.t.~$\cM$, and assume $\gamma \leq 1/2$.
	Then for any $n\geq 1$, 
	there exists a choice of parameters $K, N, \eta, \epsexpl$ such that
 \cref{alg:sc_sdpo} executed with the $L^1$ action norm guarantees for any $\delta>0$, that w.p.~$\geq 1-\delta$ the total number of environment time steps $\leq n$, and the output policy satisfies 
	\begin{align*}
        V(\hat \pi) - V^\star(\Pi)
        &= O\br{
        \frac{\cvgd^2 A^2 H^{7} \sqrt{\log\frac{ n \cC(\Pi)}{\delta}}}{ n^{2/15}}
        + \cvgd H^3 \sqrt A n^{1/30}\sqrt{\epserm}
        + \epsvgd},
    \end{align*}
    where  $\cC(\Pi) \eqq \cN(\epsnet, \Pi, \norm{\cdot}_{\infty, 1})$ is the $\epsnet$-covering number of $\Pi$
    and 
   	\begin{aligni*}
    	\epsnet 
    	= \Omega\br{\tfrac{ 1}{ \log(n/\delta) n }}
    	.
    \end{aligni*}
\end{theorem}
\begin{proof}
	Note that for any $p,q\in \Delta(\cA)$, $\frac12\norm{p-q}_1^2\leq 2$, and
	\begin{align*}
		\frac12\norm{p-p_0}_1^2
		- \frac12\norm{q-p_0}_1^2
		&=
		\frac12\br{\norm{p-p_0}_1 + \norm{q-p_0}_1}
		\br{\norm{p-p_0}_1 - \norm{q-p_0}_1}
		\\
		&\leq2\br{\norm{p-p_0}_1 - \norm{q-p_0}_1}
		\\
		&\leq2\norm{p-q}_1.
	\end{align*}
	Thus, when executing \cref{alg:sc_sdpo} over $K$ iterations, we are in the setting of \cref{lem:update_steps_ge} with $D=2, L=2$.
	Suppose we run the algorithm for $K$ iterations with $\eta=\sqrt{\epsexpl}/(2H^3 \sqrt A)$, $\epsexpl=H^2/K^{2/3}$, and $N$ (and $K$) to be chosen later on.
	For any $\delta>0$, we have by \cref{lem:update_steps_ge} that w.p.~$\geq 1-\delta$, for all $k\in [K]$ it holds that:
	\begin{align}\label{eq:sdpo_sc_proof_1}
        \pi^{k+1}
        \in 
        \argmineps_{\pi\in \Pi^\epsexpl}\cbr{
        \Phi_k(\pi) \eqq
        \E_{s\sim \mu^k}\sbr{
        \abr{H Q^k_s, \pi_s}+ \frac1{2\eta} \norm{\pi_s - \pi_s^k}_1^2
        }
        },
    \end{align}
    with $\epserr = \epserm + \epsgen$, where
    \begin{align*}
        \epsgen =O\br{ \frac{A H^2}{\eta}\sqrt{
            \frac{\log\frac{ K N \cC(\Pi)}{\delta}}{ N}
        }},
    \end{align*}
	$\cC(\Pi) \eqq \cN(\epsnet, \Pi, \norm{\cdot}_{\infty, 1})$,
    and 
    \begin{aligni*}
    	\epsnet 
    	= \Omega\br{\tfrac{ 1}{ \log(KN/\delta) \sqrt{ N} }}.
    \end{aligni*}
    Now, \cref{eq:sdpo_sc_proof_1} implies \cref{alg:sc_sdpo} is an instance of the idealized SDPO \cref{alg:sdpo} with the error $\epserr$ defined above.
    Hence, by \cref{thm:sdpo}, we have that
    \begin{align}\label{eq:sc_sdop_1}
        V(\hat \pi) - V^\star(\Pi)
        &= O\br{
        \frac{ \cvgd^2 A H^{4}}{ K^{2/3}}
        + \cvgd H^3 \sqrt A K^{1/6}\sqrt{\epserr}
        + \epsvgd}
        \nonumber
        \\
        &= O\br{
        \frac{ \cvgd^2 A H^{4}}{ K^{2/3}}
        + \cvgd H^3 \sqrt A K^{1/6}\sqrt{\epsgen}
        + \cvgd H^3 \sqrt A K^{1/6}\sqrt{\epserm}
        + \epsvgd}.
    \end{align}
    We focus on the first two terms to choose $K,N$ as a function of $n$.
    Observe:
    \begin{align*}
        \frac{\cvgd^2 A H^{4}}{ K^{2/3}}
        + \cvgd H^3 \sqrt A K^{1/6}\sqrt{\epsgen}
        &\approx 
        \frac{\cvgd^2 A H^{4}}{ K^{2/3}}
        + \cvgd H^3 \sqrt A K^{1/6}
        	\frac{\sqrt A H}{\sqrt{\eta}}
            \frac{\log^{1/4}\frac{ K N \cC(\Pi)}{\delta}}{ N^{1/4}}
       \\
       &=
       \frac{\cvgd^2 A H^{4}}{ K^{2/3}}
        + \frac{\cvgd A H^4 K^{1/6} \iota}{\sqrt{\eta} N^{1/4}}
    \end{align*}
    with $\iota \eqq \log^{1/4}\frac{ K N \cC(\Pi)}{\delta}$.
    Further, by our choice of $\eta, \epsexpl$, $\eta=H/(2H^3 \sqrt A K^{1/3})$, hence
    \begin{align*}
    	\frac{\cvgd^2 A H^{4}}{ K^{2/3}}
        + \frac{\cvgd A H^4 K^{1/6} \iota}{\sqrt{\eta} N^{1/4}}
        &\approx
            	\frac{\cvgd^2 A H^{4}}{ K^{2/3}}
        + \frac{\cvgd A H^4 K^{1/6} \iota \sqrt{H^3 \sqrt A K^{1/3}}}{\sqrt{H} N^{1/4}}
        \\
        &\leq
         \frac{\cvgd^2 A H^{6}}{ K^{2/3}}
        + \frac{\cvgd A^2 H^5 \iota }{N^{1/4}}K^{1/3}.
    \end{align*}
    Choosing $N=K^{4}$ gives
    \begin{align*}
    	\frac{\cvgd^2 A H^{6}}{ K^{2/3}}
        + \frac{\cvgd A^2 H^5 \iota }{N^{1/4}}K^{1/3}
        &\lesssim
        \frac{\cvgd^2 A^2 H^{6} \iota}{ K^{2/3}},
    \end{align*}
    with $n \leq KN \widetilde H = K^5 \widetilde H$ with $\widetilde H= H \log(2KN/\delta)$ by \cref{lem:update_steps_ge}.
    Hence $K\geq n^{1/5}/\widetilde H^{1/5}$, and
    \begin{align*}
    	\frac{\cvgd^2 A^2 H^{6} \iota}{ K^{2/3}}
    	&\leq 
    	\frac{\cvgd^2 A^2 H^{6} \widetilde H^{2/15}\iota}{ n^{2/15}}.
    \end{align*}
	Now substitute $\widetilde H^{2/15} \lesssim H^{1/2} \log^{1/4} (n/\delta)$,
	and $\iota=\log^{1/4}\frac{ K N \cC(\Pi)}{\delta} 
	\lesssim \log^{1/4}\frac{ n \cC(\Pi)}{\delta}$
	to obtain 
	\begin{align*}
    	\frac{\cvgd^2 A^2 H^{6} \widetilde H^{2/15}\iota}{ n^{2/15}}
    	\lesssim
    	\frac{\cvgd^2 A^2 H^{7} \sqrt{\log\frac{ n \cC(\Pi)}{\delta}}}{ n^{2/15}}.
    \end{align*}
    Finally, using that and plugging our upper bound back into \cref{eq:sc_sdop_1} and using that $K \leq n^{1/5}$:
    \begin{align*}
        V(\hat \pi) - V^\star(\Pi)
        &= O\br{
        \frac{\cvgd^2 A^2 H^{7} \sqrt{\log\frac{ n \cC(\Pi)}{\delta}}}{ n^{2/15}}
        + \cvgd H^3 \sqrt A K^{1/6}\sqrt{\epserm}
        + \epsvgd}
        \\
        &= O\br{
        \frac{\cvgd^2 A^2 H^{7} \sqrt{\log\frac{ n \cC(\Pi)}{\delta}}}{ n^{2/15}}
        + \cvgd H^3 \sqrt A n^{1/30}\sqrt{\epserm}
        + \epsvgd},
    \end{align*}
    with 
   	\begin{aligni*}
    	\epsnet 
    	= \Omega\br{\tfrac{1}{ \log(KN/\delta) \sqrt{ N} }}
    	= \Omega\br{\tfrac{ 1}{ \log(n/\delta) \sqrt{ N}}}
    	= \Omega\br{\tfrac{ 1}{ \log(n/\delta) n^{2/5} }}
    	.
    \end{aligni*}
\end{proof}

\subsection{CPI and DA-CPI}
\label{sec:sc_cpi_da_cpi}
In this section we present the learning versions of CPI (\cref{alg:sc_cpi}) and DA-CPI (\cref{alg:sc_da_cpi}) and their sample complexity guarantees.
We first state two lemmas which play the same role of \cref{lem:update_steps_ge}. The proofs follow from identical arguments as those of \cref{lem:update_steps_ge}, and are thus omitted.
\begin{algorithm}[ht!]
    \caption{CPI in the learning setup} 
    \label{alg:sc_cpi}
	\begin{algorithmic}
	    \STATE \textbf{input:} Initial policy $\pi^1\in\Pi$, 
	    $\cbr{\eta_k}_{k=1}^K, N\geq 1, \epserm>0$
           \FOR{$k=1,2, \ldots,K$}
           \STATE Rollout $\pi^k$ for $N$ episodes via \cref{alg:sampler}, obtain $\cD_k = \cbr{s_i^k, \widehat Q_{s_i^k}^k}_{i=1}^N$.
           \STATE Update
           \begin{aligni*}
                \tilde \pi^{k+1}
                \gets 
                \argmineps[\epserm]_{\pi\in \Pi}
                \cbr{
                	\widehat \Phi_k(\pi)
                	\eqq \frac1N \sum_{s\in \cD_k}
                	\abr{H \widehat Q^{k}_{s}, \pi_s - \pi^k_s} 
                }
           \end{aligni*}
			\STATE Set
                $\pi^{k+1}  = (1-\eta_k) \pi^k + \eta_k \tilde \pi^{k+1}$
            \ENDFOR
            \RETURN $\hat \pi = \pi^{K+1}$
	\end{algorithmic}
\end{algorithm}

\begin{algorithm}[ht!]
    \caption{DA-CPI in the learning setup} 
    \label{alg:sc_da_cpi}
	\begin{algorithmic}
	    \STATE \textbf{input:} $
	    	\eta_1, \ldots, \eta_K > 0
	    $; $N \geq 1, \epsexpl>0$, $\epserm>0$, action norm $\norm{\cdot}_\circ$.
            \FOR{$k=1, \ldots, K$}
            \STATE Rollout $\pi^k$ for $N$ episodes via \cref{alg:sampler}, obtain $\cD_k = \cbr{s_i^k, \widehat Q_{s_i^k}^k}_{i=1}^N$.
           \STATE Update
           \begin{aligni*}
                \tilde \pi^{k+1}
                \gets 
                \argmineps[\epserm]_{\pi\in \Pi^\epsexpl}\cbr{
                \widehat \Phi_k(\pi) \eqq
                \frac1N \sum_{s \in \cD_k} {
                    \abr{\widehat Q^{k}_{s}, \pi_s} 
                }}
           \end{aligni*}
           \STATE Rollout $\pi^k$ for another $N$ episodes via \cref{alg:sampler}, obtain $\widetilde \cD_k$.
                \STATE Update
                $\pi^{k+1} \gets 
                \argmineps[\epserm]_{\pi\in \Pi^\epsexpl} \cbr{
                    \widehat \psi_k(\pi) \eqq \frac1N \sum_{s \in \widetilde \cD_k}
                    \norm{\pi_s - ((1-\eta_k) \pi^k_s + \eta_k \tilde \pi^{k+1}_s)}_\circ^2)
                }
                $
            \ENDFOR
            \RETURN $\hat \pi = \pi^{K+1}$
	\end{algorithmic}
\end{algorithm}

\begin{lemma}\label{lem:update_steps_ge_cpi}
    For $\gamma\leq 1/2$, upon execution of \cref{alg:sc_cpi}, 
    for any $\delta>0$, 
    w.p.~$\geq 1-\delta$, it holds that for all $k\in [K]$,
    \begin{align*}
        \tilde \pi^{k+1}
        &\in 
        \argmineps_{\pi\in \Pi}{
        \E_{s\sim \mu^k}\sbr{
        \abr{H Q^k_s, \pi_s - \pi^k_s}
        }},
    \end{align*}
    where $\epserr = \epserm + \epsgen$,
    \begin{aligni*}
        \epsgen \eqq \cgen A H^2 D\sqrt{
            \frac{\log\frac{KN\cN(\epsnet, \widetilde \Pi, \norm{\cdot}_{\infty, 1})}{\delta}}{ N}
        },
    \end{aligni*}
    $\cgen>0$ is an absolute numerical constant, 
    and 
    $\epsnet \geq 
    \frac{\cgen }{6 \sqrt{ N} (\log(2 KN/\delta)}$. Furthermore, the number of time steps of each episode rolled out by \cref{alg:sampler} is $\leq 2 H \log\br{2 K N/\delta} $.
\end{lemma}

\begin{lemma}\label{lem:update_steps_ge_da_cpi}
	For $\gamma\leq 1/2$, 
    upon execution of \cref{alg:sc_da_cpi}, 
    for any $\delta>0$, 
    w.p.~$\geq 1-\delta$, it holds that for all $k\in [K]$,
    \begin{align*}
        \tilde \pi^{k+1}
        &\in 
        \argmineps_{\pi\in  \Pi^\epsexpl}{
        \E_{s\sim \mu^k}\sbr{
        \abr{H Q^k_s, \pi_s - \pi^k_s}
        }}
        \\
        \pi^{k+1} &\in  
                \argmineps_{\pi\in \Pi^\epsexpl} 
                    \E_{s \sim \mu^k} 
                    \norm{\pi_s - ((1-\eta_k) \pi^k_s + \eta_k \tilde \pi^{k+1}_s)}_1^2
       ,
    \end{align*}
    where $\epserr = \epserm + \epsgen$,
    \begin{aligni*}
        \epsgen \eqq \cgen A H^2 D\sqrt{
            \frac{\log\frac{KN\cN(\epsnet, \widetilde \Pi, \norm{\cdot}_{\infty, 1})}{\delta}}{ N}
        },
    \end{aligni*}
    $\cgen>0$ is an absolute numerical constant, 
    and 
    $\epsnet \geq 
    \frac{\cgen }{6 \sqrt{ N} (\log(2 KN/\delta)}$. Furthermore, the number of time steps of each episode rolled out by \cref{alg:sampler} is $\leq 2 H \log\br{2 K N/\delta} $.
\end{lemma}

We now give the sample complexity guarantees of CPI and DA-CPI; for simplicity, we present the bounds assuming $\epserm=0$.
\begin{theorem}
	Let $\Pi$ be a policy class that satisfies $(\cvgd, \epsvgd)$-VGD w.r.t.~$\cM$, and assume $\gamma \leq 1/2$.
	Then for any $n\geq 1$, 
	there exists a choice of parameters $K, N, \cbr{\eta_k}$ such that
 \cref{alg:sc_cpi} executed with $\epserm=0$ and the $L^1$ action norm, guarantees for any $\delta>0$, that w.p.~$\geq 1-\delta$ the total number of environment time steps $\leq n$, and the output policy satisfies 
	\begin{align*}
        V(\hat \pi) - V^\star(\Pi)
        &= O\br{\frac{\cvgd^2 A H^4 \log(n\cC(\Pi)/\delta)}{n^{1/3}}
        + \epsvgd},
    \end{align*}
    where  $\cC(\Pi) \eqq \cN(\epsnet, \Pi, \norm{\cdot}_{\infty, 1})$ is the $\epsnet$-covering number of $\Pi$
    and $\epsnet = \Omega\br{\tfrac{ 1}{ \log(n/\delta) n}}$
    	.
\end{theorem}
\begin{proof}
	By \cref{lem:update_steps_ge_cpi},
	we have that w.p.~$\geq 1-\delta$, for all $k\in [K]$ it holds that:
	\begin{align*}
        \pi^{k+1}
        \in 
        \argmineps_{\pi\in \Pi}\cbr{
        \Phi_k(\pi) \eqq
        \E_{s\sim \mu^k}\sbr{
        \abr{H Q^k_s, \pi_s - \pi^k_s}
        }
        },
    \end{align*}
    with $\epserr = \epserm + \epsgen$, where
    \begin{align*}
        \epsgen =O\br{ A H^2\sqrt{
            \frac{\log\frac{ K N \cC(\Pi)}{\delta}}{ N}
        }},
    \end{align*}
	$\cC(\Pi) \eqq \cN(\epsnet, \Pi, \norm{\cdot}_{\infty, 1})$,
    and 
    \begin{aligni*}
    	\epsnet 
    	= \Omega(\tfrac{ 1}{ \log(KN/\delta) \sqrt{ N} }).
    \end{aligni*}
    Hence, \cref{alg:sc_cpi} is an instance of the idealized CPI \cref{alg:cpi} with the error $\epserr$ defined above.
    Now, by \cref{thm:cpi}, we have that
    \begin{align*}
    	V(\hat \pi) - V^\star(\Pi)
    	&= O\br{
    	\frac{\cvgd^2 H^3}{K} + \cvgd\varepsilon
        + \epsvgd
    	}
    	\\
    	&= O\br{
    	\frac{\cvgd^2 H^3}{K} + \cvgd A H^2\sqrt{
            \frac{\log\frac{ K N \cC(\Pi)}{\delta}}{ N}
        }
        + \epsvgd
    	}.
    \end{align*}
	Choosing $N = K^2$, and noting that $n\lesssim N^{3/2} H \log(n/\delta)$, we obtain
	\begin{align*}
    	V(\hat \pi) - V^\star(\Pi)
    	&= O\br{
    	\frac{\cvgd^2 A H^4 \log(n\cC(\Pi)/\delta)}{n^{1/3}}
        + \epsvgd
    	},
    \end{align*}
	which completes the proof.
\end{proof}

\begin{theorem}
	Let $\Pi$ be a policy class that satisfies $(\cvgd, \epsvgd)$-VGD w.r.t.~$\cM$, and assume $\gamma \leq 1/2$.
	Then for any $n\geq 1$, 
	there exists a choice of parameters $K, N, \epsexpl, \cbr{\eta_k}$ such that
 \cref{alg:sc_da_cpi} executed with with $\epserm=0$ and the $L^1$ action norm, guarantees for any $\delta>0$, that w.p.~$\geq 1-\delta$ the total number of environment time steps $\leq n$, and the output policy satisfies 
	\begin{align*}
        V(\hat \pi) - V^\star(\Pi)
        &= O\br{
			\frac{ \cvgd^2 A^2H^5 \sqrt{\log(n\cC(\Pi)/\delta)}}{n^{2/15} } 
        +\epsvgd
        },
    \end{align*}
    where  $\cC(\Pi) \eqq \cN(\epsnet, \Pi, \norm{\cdot}_{\infty, 1})$ is the $\epsnet$-covering number of $\Pi$
    and $\epsnet = \Omega\br{\tfrac{ 1}{ \log(n/\delta) n}}$
    	.
\end{theorem}
\begin{proof}
	By \cref{lem:update_steps_ge_da_cpi},
	we have that w.p.~$\geq 1-\delta$, sub-optimality $\epserr$ holds for all $k\in [K]$ and for both update steps
	with $\epserr = \epserm + \epsgen$, where
    \begin{align*}
        \epsgen =O\br{ A H^2\sqrt{
            \frac{\log\frac{ K N \cC(\Pi)}{\delta}}{ N}
        }},
    \end{align*}
	$\cC(\Pi) \eqq \cN(\epsnet, \Pi, \norm{\cdot}_{\infty, 1})$,
    and 
    \begin{aligni*}
    	\epsnet 
    	= \Omega(\tfrac{ 1}{ \log(KN/\delta) \sqrt{ N} }).
    \end{aligni*}
    Hence, \cref{alg:sc_da_cpi} is an instance of the idealized DA-CPI \cref{alg:cpi} with the error $\epserr$ defined above.
    Now, by \cref{thm:cpi}, we have that
    \begin{align*}
        V(\hat \pi) - V^\star(\Pi)
        &= O\br{\cvgd^2 AH^3\br{
	   \frac{1}{K^{2/3}} 
            + \varepsilon^{1/3}
    	+ \varepsilon^{2/3}  K^{2/3}}
    	+\epsvgd
        }
    .
    \end{align*}
    We focus on the first two terms; when balancing them, the third will be of the same order.
    Let $\iota\eqq\sqrt{\log(KN\cC(\Pi)/\delta)}$, and we have:
    \begin{align*}
	   \frac{1}{K^{2/3}} 
            + \varepsilon^{1/3}
       &\lesssim 
       \frac{1}{K^{2/3}} 
            + \frac{(AH^2\iota)^{1/3}}{N^{1/6}}
       \\
       &\lesssim 
       \frac{(AH^2\iota)^{1/3}}{K^{2/3}} \tag{$N=K^4$}
       \\
       &\lesssim 
       \frac{AH^2 \iota}{n^{2/15}} \tag{$n\lesssim K^5 H \log(KN/\delta)$}.
    \end{align*}
    This implies
    \begin{align*}
        V(\hat \pi) - V^\star(\Pi)
        &= O\br{\frac{ \cvgd^2 A^2H^5 \iota}{n^{2/15} } 
        +\epsvgd
        }
    ,
    \end{align*}
	which completes the proof.
\end{proof}

\subsection{PMD}
\label{sec:sc_pmd}
PMD in the learning setup is given in \cref{alg:sc_pmd}.
Since $L^2$-PMD and $L^2$-SDPO coincide (they are the exact same algorithm),
the sample complexity of PMD with the Euclidean action regularizer follows form arguments identical to those given for SDPO in the proof of \cref{thm:sc_sdpo}, but with the $L^2$-action norm instead of the $L^1$-norm. This leads to slightly worse dependence on the action set cardinality $A$, but otherwise to the same guarantee.
\begin{theorem}\label{thm:sc_pmd}
	Let $\Pi$ be a convex policy class that satisfies $(\cvgd, \epsvgd)$-VGD w.r.t.~$\cM$, and assume $\gamma \leq 1/2$.
	Then for any $n\geq 1$, 
	there exists a choice of parameters $K, N, \eta, \epsexpl$ such that
 \cref{alg:sc_pmd} executed with $\epserm=0$ and the $L^2$ action regularizer guarantees for any $\delta>0$, that w.p.~$\geq 1-\delta$ the total number of environment time steps $\leq n$, and the output policy satisfies 
	\begin{align*}
        V(\hat \pi) - V^\star(\Pi)
        &= O\br{
        \frac{\cvgd^2 A^3 H^{7} \sqrt{\log\frac{ n \cC(\Pi)}{\delta}}}{ n^{2/15}}
        + \epsvgd},
    \end{align*}
    where  $\cC(\Pi) \eqq \cN(\epsnet, \Pi, \norm{\cdot}_{\infty, 1})$ is the $\epsnet$-covering number of $\Pi$
    and 
   	\begin{aligni*}
    	\epsnet 
    	= \Omega\br{\tfrac{ 1}{ \log(n/\delta) n }}
    	.
    \end{aligni*}
\end{theorem}
Notably, the analysis of $L^2$-PMD through the SDPO perspective gives better sample complexity than we would have obtained through the Bregman-proximal method analysis \cref{thm:breg_prox}; roughly speaking this is the case because analysis based on the Bregman divergence hinges on approximate optimality conditions rather than sub-optimality in function values. This leads to dependence of $\epsgen^{1/4}$ rather than $\sqrt \epsgen$, which further leads to inferior sample complexity.
For action regularizers other than $L^2$, sample complexity of PMD may be derived again using similar arguments to those of \cref{thm:sc_sdpo} but now combined with \cref{thm:breg_prox}. A little more care is needed in the choice of parameters, as smoothness of the action regularizer (needed both in \cref{thm:breg_prox} and in \cref{lem:update_steps_ge}) is commonly inversely related to $\epsexpl$. As a result, with the currently known techniques for the iteration complexity upper bound, the sample complexity upper bounds for essentially any action regularizer other than $L^2$, will be worse than those of the $L^2$ case w.r.t.~dependence on both $n$ and $A$.

\begin{algorithm}[tb]
   \caption{PMD in the learning setup}
   \label{alg:sc_pmd}
\begin{algorithmic}
   \STATE {\bfseries Input:} 
   $K\geq 1, N\geq 1, \eta > 0, \epsexpl >0, \epserm>0, \Pi\in \Delta(\cA)^\cS$, and action regularizer $R \colon\R^\cA \to \R$.
   \STATE Initialize $\pi^1 \in \Pi^\epsexpl$
   \FOR{$k=1$ {\bfseries to} $K$}
   \STATE Rollout $\pi^k$ for $N$ episodes via \cref{alg:sampler}, obtain $\cD_k = \cbr{s_i^k, \widehat Q_{s_i^k}^k}_{i=1}^N$.
   \STATE Update
   \begin{aligni*}
        \pi^{k+1}
        \gets 
        \argmineps[\epserm]_{\pi\in \Pi^\epsexpl}\cbr{
        \widehat \Phi_k(\pi) \eqq
        \frac1N \sum_{s \in \cD_k} {
            \abr{H \widehat Q^{k}_{s}, \pi_s} 
            + \frac1{\eta} B_R{\pi_s, \pi_s^k}
        }}
   \end{aligni*}
   \ENDFOR
   \RETURN $\hat \pi \eqq \pi^{K+1}$
\end{algorithmic}
\end{algorithm}

%\begin{theorem}\label{thm:da_cpi}
%    Let $\Pi$ be policy class that satisfies $(\cvgd, 0)$-VGD w.r.t.~$\cM$.
%    Suppose that DA-CPI (\cref{alg:da_cpi}) is executed with the l1 action norm $\norm{\cdot}_1$.
%    Then, for any $N\geq 1$, with an appropriate setting of $\epsilon, \tilde\epsilon, K, \cbr{\eta_k}$ and $\epsexpl$, we have w.p.~$\geq1-\delta$, that the output $\hat \pi \in \Pi^\epsexpl$ of SDPO satisfies:
%    \begin{align*}
%        V(\hat \pi) - V^\star(\Pi)
%        &= O\br{
%    	\frac{\cvgd^2 H^4 A \log^{1/4}(n \cC(\Pi)/\delta)}{n^{2/15}}
%		+ {\cvgd  A H^4 n^{1/5}} \sqrt{\epserm} }
%    ,
%    \end{align*}
%    with 
%    $\cC (\Pi) \eqq \cN(1/(H n^{4/15}\log(n/\delta)), \Pi, \norm{\cdot}_{\infty, 1})$.
%\end{theorem}
%
%\begin{theorem}\label{thm:pmd}
%    Let $\Pi$ be policy class that satisfies $(\cvgd, 0)$-VGD w.r.t.~$\cM$.
%    Suppose that PMD (\cref{alg:pmd}) is executed with ...
%    Then, for any $N\geq 1$, with an appropriate setting of $K, \eta$ and $\epsexpl$, we have w.p.~$\geq1-\delta$, that the output $\hat \pi \in \Pi^\epsexpl$ of PMD satisfies:
%    \begin{align*}
%        V(\hat \pi) - V^\star(\Pi)
%        &= O\br{???},
%    \end{align*}
%    with 
%    $\cC (\Pi) \eqq \cN(1/(N^{5/8}\log(N/\delta)), \Pi, \norm{\cdot}_{\infty, 1})$.
%\end{theorem}

\subsection[Proof of good event lemma]{Proof of \cref{lem:update_steps_ge}}
\label{sec:update_steps_ge_proof}
\begin{lemma}\label{lem:sampler}
    \cref{alg:sampler}
    returns $s_t, \widehat Q^\pi_{s_t}$ that satisfy
    (i) $s_t\sim \mu^\pi$; 
    (ii) $\E \sbr[b]{ \widehat Q^\pi_{s_t} \mid s_t} = Q^\pi_{s_t}$.
    Furthermore, for any $\delta>0$, w.p.~$\geq 1-\delta$, 
    the algorithm terminates after no more than $\frac{2}{1-\gamma}\log \frac2{\delta}$ time steps and it holds that
    \begin{aligni*}
    \av{\widehat Q^\pi_{s_t, a_t}} \leq HA\log(2/\delta).
    \end{aligni*}
    % \begin{aligni*}
    % \E_{a \sim \pi_{s_t}}\norm[b]{\widehat Q^\pi_{s_t}}^2_1 \leq \fraci{\br{H\log(1/\delta)}^2}{\pi_{s_t, a_t}}
    % \end{aligni*}
\end{lemma}
\begin{proof}[Proof of \cref{lem:sampler}]
	First, note that (i) follows directly from the definition of the discounted occupancy measure \cref{def:occmeas_rho0}. 
    Indeed, by definition of \cref{alg:sampler}, for any $s\in \cS$:
    \begin{align*}
        \Pr(s \text{ is accepted by \cref{alg:sampler} on step } t)
        &= \gamma^{t}(1-\gamma)\Pr(s_t = s \mid \pi, s_0\sim \rho_0)
        \\
        \implies
        \Pr(s \text{ is accepted by \cref{alg:sampler}})
        &= (1-\gamma)\sum_{t=0}^\infty\gamma^{t}
        \Pr(s_t = s \mid \pi, s_0\sim \rho_0)
        =\mu^\pi(s).
    \end{align*}
    Further, let $\mathbf {alg}$ denote \cref{alg:sampler}, and then 
    \begin{align*}
    \E_{\mathbf {alg}}\sbr{\widehat Q^\pi_{s_t, a_t} \mid s_t, a_t}
    &= 
    A\E_{\mathbf {alg}}\sbr{\sum_{t'=t}^T r(s_{t'}, a_{t'}) \mid s_t, a_t}
    \\
    &= A\sum_{t'=t}^\infty\gamma^{t-t'} 
    \E_\pi \sbr{r(s_{t'}, a_{t'}) \mid s_t, a_t}
    \\
    &= A\sum_{t'=t}^\infty\gamma^{t-t'} 
    \E_\pi \sbr{r(x_h, u_h) \mid x_0=s_t, u_0=a_t},
    \end{align*}
    where in the last expression, expectation is w.r.t.~$x_{h+1}\sim \P(\cdot|x_h, u_h)$ for $h\geq1$, and $u_h \sim \pi(\cdot|x_h)$ for $h\geq 2$.
    Now,
    \begin{align*}
        \sum_{t'=t}^\infty\gamma^{t-t'} 
    \E_\pi \sbr{r(x_h, u_h) \mid x_0=s_t, u_0=a_t}
    &=\sum_{h=0}^\infty\gamma^{h} 
    \E_\pi \sbr{r(x_h, u_h) \mid x_0=s_t, u_0=a_t}
    \\
    &=\E_\pi \sbr{\sum_{h=0}^\infty\gamma^{h} 
        r(x_h, u_h) \mid x_0=s_t, u_0=a_t}
    \\
    &= Q_{s_t, a_t}^\pi,
    \end{align*}
    hence for all $a\in \cA$,
    \begin{align*}
    \E_{\mathbf {alg}}\sbr{\widehat Q^\pi_{s_t, a} \mid s_t}
    =
    \Pr(a_t = a)\E_{\mathbf {alg}}\sbr{\widehat Q^\pi_{s_t, a_t} \mid s_t, a_t=a}
    = \Pr(a_t = a) A Q_{s_t, a}^\pi
    =
    Q_{s_t, a}^\pi,
    \end{align*}
    which proves (ii).

	For the second part, observe that the
    acceptance event occurs w.p.$1-\gamma$ at each time step, therefore the probability for acceptance to not occur in the first $t$ times steps is $\gamma^t$, and hence probability of acceptance by time $t$ is $1-\gamma^t$. 
    We have $\gamma^t \leq e^{-(1-\gamma)t}\leq \delta$
    for $t\geq \frac{1}{1-\gamma}\log \frac1\delta$,
    hence, w.p.~$\geq 1-\delta/2$ acceptance occurs before time step $t_\delta \eqq \frac{1}{1-\gamma}\log \frac2{\delta}$.
    Same goes for the termination event, hence, the episode terminates after $2t_\delta$ time steps w.p. $\geq 1-\delta$.
	
	Finally, the bound on $\av[b]{\widehat Q^\pi_{s_t, a_t}}$ follows from the termination event and definition of $\widehat Q^\pi_{s_t, a_t}$ in the algorithm.
\end{proof}
Our proof makes use of the notion of sub-exponential norm \citep{vershynin2018high} of a random variable $X$:
\begin{align}
	\norm{X}_{\psi_1} \eqq \inf\cbr{\alpha > 0 : \E e^{|X|/\alpha}\leq 2}.
\end{align}

\begin{lemma}\label{lem:geom_subexp}
Assume $T$ is a geometric random variable, $T \sim {\rm Geom(p})$, i.e., $\Pr(T = t) = (1-p)^{t-1}p$ for $1\leq t\in \N$.
Then, for $q\eqq\max\cbr{p, 1-p}$:
$$
\norm{T}_{\psi_1}\leq 1/\ln\br{1+\tfrac{1-q}{4q}}.
$$	
\end{lemma}
\begin{proof}
	We have:
\begin{align*}
	\E e^{T/\alpha}
	&= \sum_{t=1}^\infty (1-p)^{t-1}p e^{t/\alpha}
	\\
	&= pe^{1/\alpha}\sum_{t=1}^\infty (1-p)^{t-1} e^{(t-1)/\alpha}
	\\
	&= pe^{1/\alpha}\sum_{t=0}^\infty \br{(1-p)e^{1/\alpha}}^t
\end{align*}
Let $q\eqq\max\cbr{p, 1-p}$, and set
$\alpha=1/\ln\br{1+\tfrac{1-q}{4q}}$, then,
\begin{align*}
	e^{1/\alpha} = 1+\frac{1-q}{4q}.
\end{align*}
Now,
\begin{align*}
	(1-p)e^{1/\alpha}
	= 1-p+(1-p)\frac{1-q}{4q}
	\leq 1-p+\frac{1-q}{4}
	\leq 1-p+\frac{p}{4},
\end{align*}
hence
\begin{align*}
	\sum_{t=0}^\infty \br{(1-p)e^{1/\alpha}}^t
	\leq \frac{1}{1-\br{1-p+\frac{p}{4}}}
	= \frac{4}{3 p}.
\end{align*}
Combining with our previous derivation, we obtain
\begin{align*}
	\E e^{T/\alpha}
	&\leq pe^{1/\alpha}\frac{4}{3 p}
	=\frac43\br{1+\frac{1-q}{4q}}
	\leq \frac43\br{1+\frac{1}{4}}
	=\frac53 \leq 2.
\end{align*}
\end{proof}

\begin{lemma}\label{lem:subexp_concentration}
	Assume $X$ is a random variable that satisfies $\norm{X}_{\psi_1}\leq R$. Then for $N$ independent samples of $X_1, \ldots, X_N$, we have for any $\epsilon\leq R$:
\begin{align*}
	\Pr\br{\av{\frac{1}{N}\sum_{i=1}^N T_i - \E T }\geq \epsilon}
	&\leq2e^{-c \frac{N \epsilon^2}{R^2} },
\end{align*}
where $c>0$ is an absolute numerical constant.
\end{lemma}
Centering only costs an absolute constant factor \citep{vershynin2018high}, thus
$\norm{X-\E x}_{\psi_1} \leq \tilde R$
for $\tilde R= c_0 R$. Now, the Bernstein inequality for sums of independent sub-exponential random variables (Theorem 2.8.1 of \citealp{vershynin2018high}) yields:
\begin{align*}
	\Pr\br{\av{\frac{1}{N}\sum_{i=1}^N X_i - \E X }\geq \epsilon}
	&\leq 
	2\exp\sbr{-c_1\min\br{
		\frac{N^2\epsilon^2}{N \tilde R^2}, \frac{N \epsilon}{\tilde R}}
	}
	\\
	&=2\exp\sbr{-c_1 N \min\br{
		\frac{\epsilon^2}{ \tilde R^2}, \frac{ \epsilon}{\tilde R}}
	}
	\\
	&=2\exp\sbr{-(c_1/c_0) \frac{\epsilon^2 N }{ R^2}}
\end{align*}
where we use the assumption that $\epsilon\leq R \leq \tilde R$.

\begin{lemma}[Empirical objective concentration]
\label{lem:phi_k_conc}
	For a given fixed $k$ and a given fixed policy $\pi$, 
	we have that for any $\delta'>0$, w.p.~$\geq1-\delta'$ it holds that:
    \begin{align*}
        \av{\Phi_k(\pi) - \widehat \Phi_k(\pi)}
        \leq \frac{C A H^2 D}{\eta}\sqrt{
        \frac{ \log\frac{2}{\delta'}}{N}
        }
        ,
    \end{align*}
    where $C>0$ is a universal numerical constant and $\Phi_k, \widehat \Phi_k$ are defined in \cref{eq:concentration_hat_phi_k,eq:concentration_phi_k} of \cref{lem:sampler}.
\end{lemma}
\begin{proof}
    Denote:
    \begin{align*}
        \hat \ell^k(\pi; s) \eqq 
            H\abr[b]{\widehat Q^k_s, \pi_s} + \frac1{\eta}\Div\br{\pi_s, \pi^k_s},
    \end{align*}
    and note that by \cref{lem:sampler}, we have
    \begin{align*}
        \Phi_k(\pi) = 
        \E_{(s, \widehat Q_s^k) \sim {\rm sampler}(\pi^k)}
        \sbr{
            \hat \ell^k(\pi; s)
        },
    \end{align*}
    where ``${\rm sampler}$'' denotes \cref{alg:sampler}.
    Now, for $(s, \widehat Q_s^k) \sim {\rm sampler}(\pi^k)$, we consider the RV
    $$
    X = \abr{\widehat Q^k_s, \pi_s} + \frac1\eta \Div(\pi_s, \pi^k_s).
    $$
    The divergence term is bounded by $D/\eta$, hence $\norm{\frac1\eta \Div(\pi_s, \pi^k_s)}_{\psi_1}\leq D/\eta$ follows immediately.
    Further, the RV $\frac1A\widehat Q_{s, a}^k$ for the action $a$ accepted in \cref{alg:sampler} is dominated by a geometric RV $T\sim {\rm Geom(p=1-\gamma)}$, therefore $\norm{H \widehat Q_{s, a}^k}_{\psi_1}
    	\leq H A \norm{T}_{\psi_1}$.
    By \cref{lem:geom_subexp} and our assumption that $\gamma \leq 1/2$,
    \begin{align*}
    	\norm{T}_{\psi_1} \leq \frac1{\ln\br{1+\tfrac{1-\gamma}{4\gamma}}}
    	\leq \frac1{\tfrac{1-\gamma}{8\gamma}}
    	\leq \frac{8}{1-\gamma} = 8 H.
    \end{align*}
	Now, again using that $\norm{\cdot}_{\psi_1}$ is a norm, we obtain:
	\begin{align*}
			\norm{H\abr{\widehat Q^k_s, \pi_s } + \frac1\eta\Div(\pi_s, \pi^k_s)}_{\psi_1}
			\leq D/\eta + 8H^2A.
	\end{align*}
	Now, by \cref{lem:subexp_concentration}, we have
	\begin{align*}
        \Pr\br{\av{\widehat \Phi_k(\pi)
        - \Phi_k(\pi)}\geq \epsilon}
        \leq 
        2 e^{- c' \frac{N \epsilon^2}{A^2 H^4 + D^2/\eta^2}}
        \leq
        2 e^{- c' \frac{N \epsilon^2 \eta^2}{ A^2 H^4 D^2}}.
    \end{align*}
    for an appropriate universal constant $c'$.
    Letting $\delta=2 e^{- c' \frac{N \epsilon^2 \eta^2}{ A^2 H^4 D^2}}$, the result follows.
%    \us{.....}
%    $s\sim \mu^k$, we have that the random variable
%    \begin{aligni*}
%        \frac1{\eta}\Div\br{\pi_s, \pi^k_s}
%    \end{aligni*}
%    is bounded in $[0, D/\eta]$, thus by Hoeffding's inequality,
%    \begin{align*}
%        \Pr\br{\av{\frac1N \sum_{s\in \cD_k}\frac1{\eta}\Div\br{\pi_s, \pi^k_s}
%        - \E_{s\sim \mu^k}\frac1{\eta}\Div\br{\pi_s, \pi^k_s}}\geq \epsilon}
%        \leq 
%        2 e^{- \frac{N \epsilon^2}{2 D^2/\eta}}
%        \leq
%        2 e^{- \frac{N \epsilon^2 \eta}{2 D^2}}.
%    \end{align*}
%    Further, we have that
%    \begin{aligni*}
%        H\abr[b]{\widehat Q^k_s, \pi_s}
%    \end{aligni*}
%    is a geometric random variable and therefore sub-exponential \citep[e.g.,][]{wainwright2019high,zhang2020concentration}.
%    Thus, by Bernstein's inequality \citep{wainwright2019high} we obtain
%    \us{TODO}
%    \begin{align*}
%        \Pr\br{\av{\widehat \Phi_k(\pi)
%        - \Phi_k(\pi)}\geq \epsilon}
%        = \Pr\br{\av{\frac1N\sum_{i=1}^n \hat \l^k(\pi; s)
%        - \Phi_k(\pi)} \geq \epsilon}
%        \leq 
%        2 e^{- \frac{N \epsilon^2}{2 v^2 D^2/\eta}}
%        \leq
%        2 e^{- \frac{N \epsilon^2 \eta}{4 A^2 H^2 D^2}}.
%    \end{align*}
\end{proof}
We are now ready for the proof of \cref{lem:update_steps_ge}.
\begin{proof}[Proof of \cref{lem:update_steps_ge}]
    Consider the ``good event'' described next.
    Let $\delta > 0$ and denote
    \begin{align*}
        \widetilde H &\eqq  H\log(2 K N/\delta),
        \\
        \epsnet &\eqq \frac{C AH^2D}{\sqrt{ N} (A \widetilde H H  + L/\eta)},
    \end{align*}
    where $C$ is specified by \cref{lem:phi_k_conc}.
    Further, let ${\rm Net}(\epsnet, \widetilde \Pi) \eqq {\rm Net}(\epsnet, \widetilde \Pi, \norm{\cdot}_{\infty, 1})$ be an $\epsnet$-cover of $\widetilde \Pi$ w.r.t.~$\norm{\cdot}_{\infty, 1}$ of size $\cN(\epsnet, \widetilde \Pi) \eqq \cN(\epsnet, \widetilde \Pi, \norm{\cdot}_{\infty, 1})$.
    Consider the following events:
    \begin{enumerate}
        \item For all $k\in [K], i\in [N]:$ $\widehat Q^k_{s_i^k, a_i^k} 
        \leq \widetilde H A
        $, and the corresponding  episode length $\leq \widetilde H$
        \item For all $k\in [K]$:
        \begin{align*}
            \forall \pi\in {\rm Net}(\epsnet, \widetilde \Pi):
    	\quad 
        |\widehat \Phi_k(\pi) - \Phi_k(\pi)| 
        \leq \frac{C A H^2 D}{\eta} \sqrt{
            \frac{\log\frac{2KN\cN(\epsnet, \widetilde \Pi)}{\delta}}{ N}
        }
        \end{align*}
    \end{enumerate}
    By \cref{lem:sampler} and the union bound, event (1) holds w.p.~$\geq1-\delta/2$, and by \cref{lem:phi_k_conc} and the union bound, event (2) holds w.p.~$\geq 1-\delta/2$.
    Hence, the good event holds w.p.~$\geq 1-\delta$.
    Proceeding, we assume the good event holds.
    We have that for any $\pi, \pi'$,
    \begin{align*}
        \av{\Phi_k(\pi) - \Phi_k(\pi')}
        &= \av{
        \E_{s\sim \mu^{\pi^k}}\sbr{
            H\abr[b]{Q^k_s, \pi_s - \pi_s'} 
            + \frac1{\eta}\br{
                \Div(\pi_s, \pi^k_s) - \Div(\pi_s', \pi^k_s)
            }
        }}
        \\
        &\leq \E_{s\sim \mu^k}\sbr{\br{H^2
        + L/\eta}\norm{\pi_s - \pi_s'}_1}
        \\
        &=
        \br{H^2 + L/\eta}\norm{\pi - \pi'}_{L^1(\mu^k),1}.
    \end{align*}
    By a similar argument, using the good event (1):
    \begin{align*}
        \av{\widehat \Phi_k(\pi) - \widehat \Phi_k(\pi')}
        &\leq 
        \br{A \widetilde H H  + L/\eta}\norm{\pi - \pi'}_{L^1(\mu^k),1}.
    \end{align*}
    Further,
    \begin{align*}
        \norm{\pi - \pi'}_{L^1(\mu^k),1}
        \leq 
        \norm{\pi - \pi'}_{\infty,1},
    \end{align*}
    hence, we have that for any $\pi\in \widetilde \Pi$, there exists $\pi'\in {\rm Net}(\epsgen, \widetilde \Pi)$ such that:
    \begin{align*}
        \av{\Phi_k(\pi) - \widehat \Phi_k(\pi)}
        &\leq 
        \av{\Phi_k(\pi) - \Phi_k(\pi')}
        +
        \av{\Phi_k(\pi') - \widehat \Phi_k(\pi')}
        +
        \av{\widehat \Phi_k(\pi') - \widehat \Phi_k(\pi)}
        \\
        &\leq 
        \frac{C A H^2 D}{ \sqrt{ N}}
        +
        \av{\Phi_k(\pi') - \widehat \Phi_k(\pi')}
        +
        \frac{C A H^2 D}{ \sqrt{ N}}
        \\
        &\leq 
        \frac{3 C A H^2 D}{\eta} \sqrt{
            \frac{\log\frac{2KN\cN(\epsilon, \widetilde \Pi)}{\delta}}{ N}
        }=:
        \epsgen/2.
    \end{align*}
    Now, let 
    $\hat \pi_\star^{k+1}=\argmin_{\pi\in \widetilde \Pi}\widehat \Phi_k(\pi)$ and 
    $\pi_\star^{k+1}=\argmin_{\pi\in \widetilde \Pi}\Phi_k(\pi)$, then we have:
    \begin{align*}
        \Phi_k(\pi^{k+1})
        &\leq \widehat \Phi_k(\pi^{k+1}) + \epsgen/2
        \\
        &\leq \widehat \Phi_k(\hat \pi^{k+1}_\star) + \epsgen/2
        + \epserm
        \\
        &\leq \widehat \Phi_k(\pi^{k+1}_\star) + \epsgen/2
        + \epserm
        \\
        &\leq \Phi_k(\pi^{k+1}_\star) + \epsgen
        + \epserm
        \\
        &= \min_{\pi\in \widetilde \Pi}\Phi_k(\pi) + \epsgen
        + \epserm
        .
    \end{align*}
    This completes the proof.
\end{proof}

\section{Experiments implementation details}
\label{sec:vgd_experiments_details}
In this section, we provide further details on the experimental setup where we evaluate the VGD condition parameters. The pseudocode used for the experiments is give in \cref{alg:vgd_eval}.
The environments we tested on consider rewards and not cost functions, therefore the code and discussion below should be understood as having the negative reward as the cost function (we opt to maintain the cost formulation here to better align with our original setup). Additionally, the environments considered are finite-horizon and the objective is undiscounted.
\begin{algorithm}[t]
   \caption{Pseudocode for VGD parameter evaluation}
   \label{alg:vgd_eval}
\begin{algorithmic}
   \STATE Initialize two actor neural networks $\Pi_{\rm actor}, \Pi_{\rm vgd}$
   \STATE Initialize $\pi^1 \in \Pi_{\rm actor}$
   \FOR{$k=1$ {\bfseries to} $K$}
   \STATE {\color{gray} // Gradient estimation phase:}
   \STATE Rollout $\pi^k$ to collect $N$ environment timesteps $\cD^k=\cbr{s^k_i}_{i=1}^N$
   
   \FOR{$s\in \cD^k, a\in \cA$}
   \STATE Rollout $n_{\rm rep}$ episodes of $\pi^k$ starting from $s, a$.
   \STATE Set $\widehat Q_{s, a}^k \gets $ average of returns from the previous step.
   \ENDFOR
   \STATE {\color{gray} // (Note: we treat each state as if it were an independent sample from $\mu^k$, even though it is not.)}
   \STATE
   \STATE {\color{gray} // Update actor:}
   
   \STATE Train $\pi^{k+1}$ for $n_{\rm epochs}$ epochs, with $n_{\rm mbs}$ mini-batches in each epoch:
   \begin{align*}
        \pi^{k+1}
        \approx 
        \argmin_{\pi\in \Pi_{\rm actor}}
        \frac1N \sum_{s\in \cD^k}
            \abr{\widehat Q^{k}_{s}, \pi_s - \pi^k_s} 
            + \frac1{2\eta} \norm{\pi_s - \pi_s^k}_2^2
   \end{align*}

    \STATE {\color{gray} // Evaluate VGD:}
   
   \STATE Initialize $\tilde \pi^{k+1} \gets \pi^{k+1}$ 
   \STATE for $n_{\rm epochs}^{\rm vgd}$ epochs, with $n_{\rm mbs}^{\rm vgd}$ mini-batches in each epoch:
   \begin{align*}
        \tilde \pi^{k+1}
        \approx 
        \argmax_{\pi\in \Pi_{\rm vgd}}
        \frac1N \sum_{s\in \cD^k}
            \abr{\widehat Q^{k}_{s}, \pi^k_s - \pi_s} 
   \end{align*}
   \STATE Estimate $\widehat H^k$ the average episode length of $\pi^k$
   \STATE Report ${\rm Grad VGD}^k$ $= \frac1N \sum_{s\in \cD^k}
            \abr{\widehat H^k\widehat Q^{k}_{s}, 
                \pi^k_s  -\tilde \pi_s^{k+1}} 
                \eqqcolon \abr{\widehat \nabla V(\pi^k), 
                \pi^k  -\tilde \pi^{k+1}}$
   \ENDFOR
   \RETURN $\hat \pi \eqq \pi^{K+1}$
\end{algorithmic}
\end{algorithm}
\paragraph{Evaluating ``Sub Optimality'' at iteration $k$.}
For each (seed, environment) combination,
after the execution was concluded, we take the maximum value attained by the actor iterates, $\widehat V^\star = \min_{k\in K} \widehat V(\pi^k)$, where $\widehat V(\pi^k)$ is estimated using rollouts during experiment execution. We made sure to run the experiments for long enough (i.e., for large enough $K$) so that the algorithm converges.
Sub optimality of iteration $k$ is then given by $\widehat V(k) - \widehat V^\star$.

\paragraph{Evaluating ``VGD Ratio'' at iteration $k$; $\cvgd_k$.}
Given the sub-optimality evaluated as described in the previous paragraph, we report $\cvgd_k$ by computing the following:
\begin{align*}
    \cvgd_k \eqq
    \frac{\widehat V(\pi^k) - \widehat V^\star}
    {\abr{\widehat \nabla V(\pi^k), \pi^k - \tilde \pi^{k+1}}}
    .
\end{align*}

\paragraph{Local optima.}
As mentioned in \cref{sec:vgd_experiments}, local optima was only an issue (i.e., the case that $\widehat V^\star \neq V^\star(\Pi_{\rm actor})$) in the MinAtar environments. We note that our analysis holds just the same under the assumption the VGD condition is satisfied for $(\cvgd, 0)$ w.r.t.~a given target value $\widehat V^\star$ which is not necessarily the in-class optimal one. This may be interpreted as an execution specific value of $\epsvgd= \widehat V^\star -V^\star(\Pi_{\rm actor})$.
Thus, while the environments in question may not satisfy VGD globally, it seems convergence behavior may nonetheless be governed by the effective VGD parameters encountered during execution.

\paragraph{Hyperparameters.}
For the VGD actor, we used the AdamW optimizer with stepsize $5e-4$, $n_{\rm epochs}^{\rm vgd}=100, n_{\rm mbs}^{\rm vgd}=4$ accross all experiments.
Below, $N=N_{\rm envs}\times N_{\rm steps}$ means $N_{\rm steps}$ where executed in $N_{\rm envs}$ in parallel.
For the actor optimizer we used the Adam optimizer with step size $\eta_{\rm opt}$ (this is the step size of the ``inner'' optimization).
\begin{itemize}
    \item \textbf{Cartpole:}
    $K=40, N=4\times500, n_{\rm epochs}=100, n_{\rm mbs}=4, n_{\rm rep}=5, \eta_{\rm opt}=2e-4, \eta=0.01$.

    \item \textbf{Acrobot:}
    $K=100, N=8\times500, n_{\rm epochs}=100, n_{\rm mbs}=4, n_{\rm rep}=20, \eta_{\rm opt}=4e-4$ with linear annealing, $\eta=0.1$.

    \item \textbf{SpaceInvaders-MinAtar:}
    $K=600, N=16\times1000, n_{\rm epochs}=4, n_{\rm mbs}=8, n_{\rm rep}=5, \eta_{\rm opt}=5e-3$ with linear annealing, $\eta=0.1$.

    \item \textbf{Breakout-MinAtar:}
    $K=200, N=16\times1000, n_{\rm epochs}=100, n_{\rm mbs}=8, n_{\rm rep}=5, \eta_{\rm opt}=5e-3$ with linear annealing, $\eta=0.1$.
    
\end{itemize}
\paragraph{Additional comments.}
\begin{itemize}
    \item The architecture of both actor models is identical to that of the purejaxrl implementation \citep{lu2022discovered}.
    \item For Breakout-MinAtar and SpaceInvaders-MinAtar, execution took approximately 90-120 minutes per seed, on an NVIDIA-RTX-A5000 GPU.
    The experiments for Cartpole and Acrobot were run on similar hardware and took under 20 minutes for all 10 seeds.
\end{itemize}

\end{document}